%% file: uai2026paper.tex
\documentclass[accepted]{uai2026} 

\usepackage[american]{babel}

\usepackage{natbib} 
\usepackage{mathtools} 
\mathtoolsset{showonlyrefs}
\usepackage{booktabs} 
\usepackage{tikz} 
\usepackage{graphicx}
\usepackage{amsmath,amssymb,amsthm}
\usepackage{enumitem}
\newtheorem{thm}{Theorem}
\newtheorem{asmp}{Assumption}

\usepackage{xcolor}
\usepackage{here} 
\newtheorem{definition}{Definition}

\usepackage{multirow}
\usepackage{graphicx}
\usepackage{subcaption}
\usepackage{algorithm}
\usepackage{algpseudocode}

\title{Generalized Distribution-Free Semi-Supervised Learning with Risk Rewrite}

\author[1]{Yushi Hirose}
\author[1]{Hiroo Irobe}
\author[1,2]{Takafumi Kanamori}

\affil[1]{%
	Department of Mathematical and Computing Science\\
	 Institute of Science Tokyo\\
}
\affil[2]{%
	RIKEN Center for Advanced Intelligence Project
}

\begin{document}
	\maketitle
	
	\begin{abstract}
		Typical semi-supervised learning (SSL) methods rely on distributional assumptions, and their performance degrades when these are violated. While PNU learning, a risk rewriting method, offers a distribution-free alternative, it is restricted to binary classification and its variance optimality remains unclear. In this paper, we propose a generalized framework that constructs unbiased risk estimators using linear combinations of component risks, subsuming PNU learning and extending to multiclass classification. We derive the minimum achievable variance, demonstrating our estimator can attain lower variance than PNU in asymmetric loss scenarios. Furthermore, we establish a generalization bound directly linking this variance reduction to improved learning performance. Based on these theoretical insights, we introduce two practical SSL methods that empirically match or outperform existing approaches on binary and multiclass benchmarks.
	\end{abstract}
	
	\section{Introduction}\label{sec:intro}
	Although machine learning models typically require substantial labeled data to achieve good performance, acquiring such labels in real-world scenarios is often difficult and expensive. In contrast, unlabeled data is cheaper to acquire. Consequently, semi-supervised learning (SSL) has been extensively studied to leverage this unlabeled data when labeled data is scarce. Existing methods generally rely on specific distributional assumptions about the data or heuristics \citep{chapelle06semi,ouali20overview, mass25self} . Typical distributional assumptions include the \textit{manifold assumption}, which posits that high-dimensional data lie on a lower-dimensional manifold, and the \textit{cluster assumption}, which assumes that the decision boundary should pass through low-density regions to separate data clusters \citep{ouali20overview}. To encode these distributional assumptions into training a classifier, two main research directions have been studied. 
	
	The first is consistency regularization, which regularizes a classifier so that the prediction does not change significantly against the perturbation of unlabeled data. This improves the smoothness of the classifier along the data manifold. Early works were developed with graph-based regularization \citep{zhou03l,belkin06mani}. Recent methods based on deep models employ regularizers to enhance model robustness against input perturbations, data augmentation and dropout \citep{miyato19vat,laine17temp,tar17mean,xie20uda}.
	
	The second is Self-Training, which jointly trains a classifier and assigns labels (pseudo labels) to unlabeled data using current predictions. It pushes the decision boundary into low-density regions by discouraging low-confidence predictions \citep{ouali20overview}. The pseudo-labeling process has traditionally been managed by fixed confidence thresholds \citep{yaro95uns}, and recent works have adjusted them with training progress \citep{zhang21flex,wang23free}. These methods are refined by multiple classifiers \citep{blum98co,zhi05tri,pham21meta} to derive more accurate pseudo labels. Please refer to \citet{mass25self,ouali20overview} for other SSL methods, which are generally based on the distributional assumptions mentioned above.
	
	While these approaches have achieved remarkable success, their performance is often contingent on whether the underlying distributional assumptions hold for the target dataset. If these assumptions are violated, the inductive bias introduced by the unlabeled data can behave adversely, potentially degrading performance below that of supervised learning \citep{coz03semi,li11tow,krij17rob,wang22deb,arazo20pseudo}.

	\textbf{Unbiased Risk Estimators and Risk Rewriting.} A distinct line of research focuses on utilizing unlabeled data for risk evaluation without relying on restrictive distributional assumptions. This is achieved through \textit{risk rewriting} \citep{sugiyama22ws}, which reformulates the standard expected risk into an unbiased estimator computable directly from the available data distributions (e.g., combinations of labeled and unlabeled data). For instance, Positive-Unlabeled (PU) learning trains a classifier using only positive and unlabeled data by constructing such a rewritten risk estimator \citep{du14analysis}. \citet{sakai17semi} extended this to the semi-supervised setting by proposing PNU learning, which combines PN (supervised), PU, and NU (Negative-Unlabeled) risks. They theoretically demonstrated that PNU learning can reduce the variance of risk estimators compared to standard supervised learning. This approach is compelling as it is theoretically supported and applicable to any model without requiring strong distributional assumptions (such as the cluster or manifold assumptions).
	
	\textbf{Limitations and Open Problems.} 
	Despite its theoretical advantages, the PNU learning framework proposed by \citet{sakai17semi} leaves several open problems. First, their method is limited to binary classification. Second, it is unclear whether PNU is variance-optimal among broader classes of unbiased risk-rewriting estimators. In this paper, we address these limitations by generalizing the method to a linear combination set of risks $S_{lin}$ and proposing new SSL methods to minimize the variance of the risk estimator.
	
	Our contributions are summarized as follows:
	\begin{itemize}
		\item We propose a generalized framework for risk-rewriting SSL that subsumes PNU risk and extends naturally to the multiclass setting. As discussed in Section \ref{sec:conclusion}, our framework has a potential impact on a broad range of SSL and weakly supervised learning problems.
		\item We derive the theoretical minimum variance achievable by the set $S_{lin}$ and show it can have a smaller minimum variance than PNU risk in the general (asymmetric) loss case. In the symmetric loss case, we show PNU risk can achieve optimal minimum variance in $S_{lin}$.
		\item We propose two novel SSL methods: (1) an iterative optimization method and (2) a data-free method under an equal covariance assumption, and empirically demonstrate that they consistently outperform or match existing SSL methods.
	\end{itemize}

	\section{Problem setting}
	\label{sec:problemsetting}

	We consider a multiclass classification problem where the input space is $\mathcal{X} \subset \mathbb{R}^d$ and the output label space is $\mathcal{Y} = \{1, \dots, k\}$. We assume the data is generated from a joint distribution $p(x, y)$ over $\mathcal{X} \times \mathcal{Y}$. Let $p_i(x) = p(x \mid y=i)$ denote the class-conditional density for class $i$, and let $\theta_i= P(y=i)$ be the class prior probability, such that $\sum_{i=1}^k \theta_i = 1$. The distribution of unlabeled data is given by the mixture model:	$p(x) = \sum_{i=1}^k \theta_i p_i(x)$. In the semi-supervised setting, we are provided with labeled datasets $\mathcal{X}_i = \{x_j^i\}_{j=1}^{n_i} \stackrel{i.i.d.}{\sim} p_i(x)$ for each class $i \in \mathcal{Y}$, and an unlabeled dataset $\mathcal{X}_U = \{x_j^U\}_{j=1}^{n_U} \stackrel{i.i.d.}{\sim} p(x)$.
	
	\textbf{Risk Definitions.} Let $g: \mathcal{X} \to \mathbb{R}^k$ be a decision function (e.g., a neural network) and $l: \mathbb{R}^k \times \mathcal{Y} \to \mathbb{R}$ be a loss function. Assume that $g$ belongs to a function class $\mathcal{G}$. The goal of standard supervised learning is to minimize the true risk $R(g) := \mathbb{E}_{(x,y) \sim p(x,y)}[l(g(x), y)]$.
	
	To investigate unbiased risk estimators, we decompose the risk into components based on the data distribution and the label used for loss evaluation. We define component risks $R_{ij}(g)$ as:
	\begin{equation}
		R_{ij}(g) := \mathbb{E}_{x \sim p_i}[l(g(x), j)],
	\end{equation} which is the expected loss of predictor $g$ calculated over the distribution of class $i$ but evaluated with respect to a fixed label $j$. Similarly, we define the risk over the unlabeled distribution with respect to label $j$ as:
	\begin{equation}
		R_{Uj}(g) := \mathbb{E}_{x \sim p(x)}[l(g(x), j)].
	\end{equation}
	Using these components, the standard supervised risk can be expressed as $R(g) = \sum_{i=1}^k \theta_i R_{ii}(g)$.
	
	\subsection{Revisiting Binary Risk Rewriting: PU, NU, and PNU Risks}
	Although the standard supervised risk is written as $R(g) = \sum_{i=1}^k \theta_i R_{ii}(g)$, it is not the only expression. In binary classification, several different formulations have been proposed for $R(g)$ with unlabeled data. Consider the case where $k=2$, with $y=1$ and $y=2$ representing the positive and negative classes, respectively. Then the standard supervised risk (PN risk) is $R_{PN}(g) := \theta_1 R_{11}(g) + \theta_2 R_{22}(g)$.
	
	\textbf{PU Risk.}  Positive-Unlabeled (PU) learning \citep{du14analysis} addresses scenarios where negative labels are unavailable.  By exploiting the density equality $\theta_2 p_2(x) = p(x) - \theta_1 p_1(x)$, one can rewrite the risk on the negative class as $\theta_2 R_{22}(g) = R_{U2}(g) - \theta_1 R_{12}(g)$. This yields the PU risk:
	\begin{equation}
		R_{PU}(g) := \theta_1 R_{11}(g) - \theta_1 R_{12}(g) + R_{U2}(g).
	\end{equation}
	
	This risk is equivalent to the PN risk, but formulated without requiring the negative distribution.
	
	\textbf{NU Risk.} By applying a symmetric transformation (substituting the positive component via $\theta_1 p_1(x) = p(x) - \theta_2 p_2(x)$), we obtain the Negative-Unlabeled (NU) risk:
	\begin{equation}
		R_{NU}(g) := \theta_2 R_{22}(g) - \theta_2 R_{21}(g) + R_{U1}(g).
	\end{equation}
	
	\textbf{PNU Risk.} Under certain conditions and as $n_U\rightarrow\infty$, PU and NU learning outperform PN learning \citep{niu16}. To exploit this advantage further, \citet{sakai17semi} proposed PNU learning\footnote{They also proposed PUNU risk, but we omit it since it was shown to be inferior to PNU risk.}, which considers the following risk that linearly combines PU, NU, and PN risks with a parameter $\eta \in [-1, 1]$:
	$$ R_{\text{PNU}}^{\eta}(g) := 
	\begin{cases} 
		R_{\text{PNPU}}^{\eta}(g) & (\eta \ge 0) \\
		R_{\text{PNNU}}^{|\eta|}(g) & (\eta < 0)
	\end{cases}
	$$ where 
	\begin{align*}
		& R_{\mathrm{PNPU}}^\eta(g):=(1-\eta) R_{\mathrm{PN}}(g)+\eta R_{\mathrm{PU}}(g), \\
		& R_{\mathrm{PNNU}}^\eta(g):=(1-\eta) R_{\mathrm{PN}}(g)+\eta R_{\mathrm{NU}}(g) .
	\end{align*}
	
	The estimation of PNU risk with empirical distributions has a smaller variance than that of PN risk, which leads to more efficient SSL. This method has several advantages: (i) no requirement of distributional assumptions unlike other SSL methods, (ii) compatibility with any loss and classification model, and (iii) minimal additional computational costs. Unlike augmentation-heavy or teacher-student SSL methods \citep{zhang21flex, pham21meta}, our methods do not require augmentation passes or extra networks.
	
	Despite these advantages, PNU learning is limited to binary classification. How to extend it to multiclass classification is not obvious. In addition, it is not clear if we can construct a risk with a smaller estimation variance than that of the PNU risk.
	
	\section{Our proposed framework: Generalized Risk Rewrite}
	\label{sec:method}
	
	In this section, we propose a generalized framework for semi-supervised learning based on risk rewriting. We introduce a broad class of unbiased risk estimators formed by linear combinations of component risks and analyze their variance properties.	
	
	\subsection{Generalized Unbiased Estimators via Linear Combinations}
	\label{subsec:generalized_estimators}
	We define the set of linear combinations of risk functionals, $S_{lin}$, as 
	\begin{align}
		S_{lin} &:= \Biggl\{ 
		R_{lin}^{\{a_{ij}\}, \{b_j\}}(\cdot) := \sum_{i,j=1}^k a_{ij}R_{ij}(\cdot) + \sum_{j=1}^k b_j R_{Uj}(\cdot) \nonumber \\
		&\Biggm| a_{ij}, b_j \in \mathbb{R}, \quad \forall g\in\mathcal{G}, \ R_{lin}^{\{a_{ij}\}, \{b_j\}}(g) = R(g)
		\Biggr\}.
	\end{align}
	The set is defined as linear combinations of component risks and unlabeled risks, which are equal to $R(g)$ for all $g\in\mathcal{G}$. It is important to note that this formulation generalizes existing risks such as PU, NU, and PNU risks. We can build unbiased estimators for each $R_{lin}^{\{a_{ij}\}, \{b_j\}}(\cdot)\in S_{lin}$ by using empirical distributions. By optimizing the parameters $a_{ij}$ and $b_j$ over the full set $S_{lin}$, we can derive a risk that has a lower estimation variance than the existing methods.
	
	\subsection{Characterization under Linear Independence}
	\label{subsec:linear_independence}
	The number of parameters that define the risks in $S_{lin}$ can be reduced under the following mild assumption.	
	\begin{asmp}[Linear Independence]
		\label{asmp:linindep}
		The risk components are linearly independent in the sense that:
		\begin{equation}
			\sum_{i,j=1}^k c_{ij} R_{ij}(g) = 0 \quad (\forall g \in \mathcal{G}) \implies \forall i,j, \ c_{ij} = 0.
		\end{equation}
	\end{asmp} 
	
	This assumption is satisfied when $\mathcal{G}$ is reasonably large and the loss function is not constrained by  linear relations. In particular, it cannot be met by symmetric loss functions (such as the 0-1 loss), because $\sum_{j=1}^k l(g(x),j)$ is constant and the risk components are linearly dependent. The symmetric-loss case is discussed in Section~\ref{sec:sym_loss}.

	Under Assumption \ref{asmp:linindep}, we can characterize the parameters that satisfy the constraint $\forall g\in\mathcal{G}, R_{lin}^{\{a_{ij}\}, \{b_j\}}(g) = R(g)$.
	
	\begin{thm}[Parametrization of $S_{lin}$]
		\label{thm:rewrite_slin} 
		Under Assumption \ref{asmp:linindep}, $S_{lin}$ can be rewritten as
		\begin{align}
			S_{lin} = &\Biggl\{ R^{\boldsymbol{a}}_{lin}(\cdot) := \sum^k_{i=1}\theta_i R_{ii}(\cdot) + \sum^k_{i,j=1}\theta_i\left(\frac{a_j}{\theta_j}-1\right)R_{ij}(\cdot) \nonumber \\
			&\quad - \sum^k_{j=1} \left(\frac{a_j}{\theta_j}-1\right)R_{Uj}(\cdot) \Biggm| \boldsymbol{a} \in\mathbb{R}^k \Biggr\}.
		\end{align}
	\end{thm}
	This theorem implies that each risk in $S_{lin}$ is parametrized by a $k$-dimensional vector.
	
	\subsection{Minimum variance calculation}
	\label{subsec:variance_analysis}
	In the following, we use a fixed $g$ and often omit the argument of the risk term, denoting $R_{ij}(g)$ as $R_{ij}$ when the context is clear. Let $\hat{R}_{lin}^{\boldsymbol{a}}$ and $\hat{R}_{ij}$ denote the empirical estimators that are defined by replacing true distributions with empirical distributions. We consider the asymptotic regime where the number of unlabeled samples $n_U \to \infty$. 
	
	Let $C_m \in \mathbb{R}^{k \times k}$ be the covariance matrix for class $m$, defined as $(C_m)_{ij} := \operatorname{Cov}_{x \sim p_m}[l(g(x), i), l(g(x), j)]$. The covariance of the empirical risk estimates is given by $\operatorname{Cov}[\hat{R}_{mi}, \hat{R}_{mj}] = \frac{1}{n_m}(C_m)_{ij}$. We define the diagonal matrix $Q = \operatorname{diag}(\theta_1^{-1}, \dots, \theta_k^{-1})$ and the shift vectors $\mathbf{d}_m = \mathbf{1} - \mathbf{e}_m$, where $\mathbf{1}$ is the all-ones vector and $\mathbf{e}_m$ is the standard basis vector. We now derive the variance of the estimator $\hat{R}_{lin}^{\boldsymbol{a}}$.
	
	\begin{thm}[Variance of $\hat{R}_{lin}^{\boldsymbol{a}}$]
		\begin{align}
			\operatorname{Var}\left[\hat{R}_{lin}^{\boldsymbol{a}}\right] 
			&= \sum_{m=1}^k \frac{\theta_m^2}{n_m} \left(Q \boldsymbol{a}-\mathbf{d}_m\right)^T C_m \left(Q \boldsymbol{a}-\mathbf{d}_m\right) \nonumber \\
			&= \boldsymbol{a}^T A \boldsymbol{a} - 2\mathbf{b}^T \boldsymbol{a} + c,
		\end{align}
		where we define the system matrices as \begin{align*}
			A &= \sum_{m=1}^k \frac{\theta_m^2}{n_m} Q C_m Q,	\quad\mathbf{b}= \sum_{m=1}^k \frac{\theta_m^2}{n_m} Q C_m \mathbf{d}_m,\\
			c &= \sum_{m=1}^k \frac{\theta_m^2}{n_m} \mathbf{d}_m^T C_m \mathbf{d}_m.
		\end{align*}
	\end{thm}
	
	Since $A$ is a sum of positive semi-definite matrices, the variance is a convex function of $\boldsymbol{a}$. Define $f(\boldsymbol{a}):=\operatorname{Var}[\hat{R}_{lin}^{\boldsymbol{a}}]$. The optimal parameter vector $\boldsymbol{a}^*$ that minimizes the variance satisfies the first-order condition:
	\begin{equation}
		\nabla_{\boldsymbol{a}} f(\boldsymbol{a}) = 2A\boldsymbol{a} - 2\mathbf{b} = 0 \implies A\boldsymbol{a}^* = \mathbf{b}.
	\end{equation}
	If $A$ is invertible, the minimum variance is given by $f(\boldsymbol{a}^*) = c - \mathbf{b}^T A^{-1} \mathbf{b}$. For simplicity, we define the total weighted covariance matrix $S$ and the weighted covariance vector $\mathbf{u}$:
	$$S := \sum_{m=1}^k \frac{\theta_m^2}{n_m} C_m, \quad \mathbf{u} := \sum_{m=1}^k \frac{\theta_m^2}{n_m} C_m \mathbf{e}_m.$$
	Then, we derive the following theorem:
	
	\begin{thm}[Minimum Variance of $\hat{R}_{lin}^{\boldsymbol{a}}$]\label{thm:min_var} Assume $n_U\rightarrow\infty$ and that $S$ is invertible. The minimum variance achievable by the linear risk estimator $\hat{R}_{lin}^{\boldsymbol{a}}$ is given by
		\begin{equation}
			\min_{\boldsymbol{a}} \operatorname{Var}\left[\hat{R}_{lin}^{\boldsymbol{a}}\right] = \sum_{m=1}^k \frac{\theta_m^2}{n_m} (C_m)_{mm} - \mathbf{u}^T S^{-1} \mathbf{u}.
		\end{equation}
	\end{thm}
	
	The first term represents the pointwise variance of the supervised estimator for the same fixed $g$, and the second term $\mathbf{u}^T S^{-1} \mathbf{u}$ (which is non-negative as $S$ is positive semi-definite) represents the variance reduction from the proposed semi-supervised method.

	\subsection{Magnitude of variance reduction}\label{subsec:var_red}
	The variance reduction $\mathbf{u}^T S^{-1} \mathbf{u}$ depends on the sample sizes, class priors and covariance matrices. To analyze when the reduction is large, we assume $\forall m\in[k], C_m=C$ where $(C)_{ii}=\rho_1>0$ and $(C)_{ij}=\rho_2<0$ for $i\neq j$, since $(C_m)_{ij}$ for $i\neq j$ tends to be negative for usual loss functions.
	\begin{thm}[Variance reduction]\label{thm:varred} Assume $C$ is invertible and $\forall m\in[k], C_m=C$. Then,
		\begin{align*}
			\mathbf{u}^T S^{-1} \mathbf{u} &= (\rho_1 - \rho_2) \frac{\sum_{m=1}^kw_m^2}{W} + \rho_2 W\\
			&\leq\rho_1 \max_{m \in [k]}w_m+\rho_2\left(W-\max_{m \in [k]}w_m\right),
		\end{align*} 
		where $w_m := \frac{\theta_m^2}{n_m}$ and $W := \sum_{m=1}^k w_m$. 
	\end{thm}
	
	Considering the upper bound of $\mathbf{u}^T S^{-1} \mathbf{u}$, the variance reduction is large when $\max_{m \in [k]}w_m$ is large and $W-\max_{m \in [k]}w_m$ is small, which means $\theta_i$ and $n_i$ should be disproportionate, i.e., there exists class imbalance in labeled samples. 
	
	\subsection{Comparison of Optimal Variances with PNU Learning}
	We compare the minimum variance achievable by our proposed method with that of PNU learning in binary classification. This analysis is different from the section~\ref{subsec:var_red}, which studies the absolute variance reduction relative to supervised learning. Here, we instead quantify the relative advantage of the proposed estimator over PNU learning. For simplicity, we compare our proposed estimator against the PNPU and PNNU estimators.

	\begin{thm}[Variance comparison]\label{thm:variance_dominance}Assume $k=2$ and $n_U\rightarrow\infty$. Let $\eta^*_{PU} :=\operatorname{argmin}_{\eta\in\mathbb{R}}\operatorname{Var}(\hat{R}_{\text{PNPU}}^{\eta})$ and $\eta^*_{NU} :=\operatorname{argmin}_{\eta\in\mathbb{R}}\operatorname{Var}(\hat{R}_{\text{PNNU}}^{\eta})$ be the optimal parameters for the PNPU and PNNU estimators respectively. Then, \begin{equation}\label{eq:var_com}\operatorname{Var}(\hat{R}_{lin}^{\boldsymbol{a}^*}) \le \min \left( \operatorname{Var}(\hat{R}_{\text{PNPU}}^{\eta^*_{PU}}), \operatorname{Var}(\hat{R}_{\text{PNNU}}^{\eta^*_{NU}}) \right).\end{equation}
		Furthermore, under the same assumption as Theorem \ref{thm:varred}, 
		\begin{align}
			\min &\left( \operatorname{Var}(\hat{R}_{\text{PNPU}}^{\eta^*_{PU}}), \operatorname{Var}(\hat{R}_{\text{PNNU}}^{\eta^*_{NU}}) \right)-\operatorname{Var}(\hat{R}_{lin}^{\boldsymbol{a}^*})\notag\\
			\label{eq:var_comval}&=\min\left(\frac{w_1^2 (\rho_1^2 - \rho_2^2)}{(w_1+w_2) \rho_1}, \frac{w_2^2 (\rho_1^2 - \rho_2^2)}{(w_1+w_2) \rho_1}\right)\end{align}
	\end{thm}

	Eq. \eqref{eq:var_com} implies our proposed estimator has a minimum variance less than or equal to that of PNPU and PNNU estimators, which is rather trivial since the set of $\hat R_{lin}^{\boldsymbol{a}}$ subsumes the sets of $\hat{R}_{\text{PNPU}}^{\eta}$ and $\hat{R}_{\text{PNNU}}^{\eta}$. Eq. \eqref{eq:var_comval} implies the gap in minimum variances becomes small when $w_1\gg w_2$ or $w_1\ll w_2$, which means our method is more advantageous when $\theta_i$ and $n_i$ are proportional, i.e., the labeled data is balanced. 	 
	 
	\section{Symmetric Loss Case}\label{sec:sym_loss}
	The preceding sections analyzed the properties of $S_{lin}$ under Assumption \ref{asmp:linindep}. For symmetric loss functions such as the 0-1 loss, Assumption \ref{asmp:linindep} does not hold,  and we need a modified analysis, as described in this section.  
	
	\subsection{Problem Setup under Symmetric Loss}
	We assume the loss function $l(\cdot, \cdot)$ satisfies the symmetric condition:
	\begin{definition}[Symmetric Loss]
		A loss function $l(\cdot,\cdot)$ is said to be \textbf{symmetric} if, for any input $x$ and any prediction vector $g(x)$, the sum of losses over all possible class labels $j \in \{1, \dots, k\}$ is constant:
		\begin{equation}
			\sum_{j=1}^k l(g(x), j) = \alpha, \quad \forall x \in \mathcal{X}, \forall g,
		\end{equation}
		where $\alpha$ is a constant.
	\end{definition}
	
	Common examples include the 0-1 loss and ramp loss \citep{char19ons}.
	This constraint implies a deterministic linear dependency among the risk components:$\sum_{j=1}^k R_{ij}(g) = \alpha$. Consequently, the linear independence assumption (Assumption \ref{asmp:linindep}) does not hold, and we need a different reparametrization for $S_{lin}$.
	
	\subsection{Reparametrization of $S_{lin}$ with symmetric loss}
	To analyze the class of risks, we introduce a modified assumption similar to linear independence.
	
	\begin{asmp}\label{asmp:linindep_sym}
		\begin{align*}
			\sum_{i}^k\sum_{j}^{k-1} &c_{ij} R_{ij}(g) +c = 0 (\forall g \in \mathcal{G})\\
			& \Rightarrow \forall i\in[k],j\in[k-1], \ c_{ij} = 0 \text{ and } c=0.
		\end{align*}
	\end{asmp}
	
	Under this assumption, we can characterize $S_{lin}$ similarly to Theorem \ref{thm:rewrite_slin}.
	
	\begin{thm}\label{thm:rewrite_slin_sym}
		Under Assumption \ref{asmp:linindep_sym}, $S_{lin}$ with a symmetric loss can be rewritten as
		\small\begin{align}
			S_{lin} = &\Biggl\{ R^{\boldsymbol{a}}_{lin}(\cdot) := \sum^k_{i=1}\theta_i R_{ii}(\cdot) + \sum^k_{i=1}\sum^{k-1}_{j=1}\theta_i\left(\frac{a_j}{\theta_j}-1\right)R_{ij}(\cdot) \nonumber \\
			&\quad -\sum^{k-1}_{j=1} \left(\frac{a_j}{\theta_j}-1\right)R_{Uj}(\cdot) \Biggm| \boldsymbol{a} \in\mathbb{R}^{k-1}\Biggr\}.
		\end{align}
	\end{thm}
	
	Thus, $S_{lin}$ with a symmetric loss is parametrized by the $k-1$ parameters, $\boldsymbol{a}_{i},i\in[k-1]$.

	\subsection{Variance of $\hat{R}_{lin}^{\boldsymbol{a}}$ with Symmetric Loss}
	Since the symmetric loss constraint implies $R_{mk} = \alpha - \sum_{j=1}^{k-1} R_{mj}$, the $k$-th row and column of the covariance matrix $C_m$ are entirely determined by the covariance of the first $k-1$ risk terms. Thus, the variance of $\hat R_{lin}^{\boldsymbol{a}}$ is expressed with truncated matrices of $C_m$. Let $\check{C}_m$ be the top-left $(k-1) \times (k-1)$ submatrix of $C_m$. 
	
	\begin{thm}[Variance of $\hat{R}_{lin}^{\boldsymbol{a}}$ with symmetric loss]\label{thm:var_sym}
		With a symmetric loss function $l(\cdot,\cdot)$,
		\begin{align*}
			\operatorname{Var}\left[\hat{R}_{lin}^{\boldsymbol{a}}\right] 
			&= \sum_{m=1}^k \frac{\theta_m^2}{n_m} \left(\check{Q} \boldsymbol{a}-\check{\mathbf{d}}_m\right)^T \check{C}_m\left(\check{Q} \boldsymbol{a}-\check{\mathbf{d}}_m\right)
		\end{align*}
		where $\check{Q} = \operatorname{diag}(\frac{1}{\theta_1}, \dots, \frac{1}{\theta_{k-1}})$ and 
		$\check{\mathbf{d}}_m = \begin{cases} 
			\mathbf{1} - \mathbf{e}_m & \text{if } m < k \\
			2\mathbf{1} & \text{if } m = k
		\end{cases}.$
	\end{thm}
	
	To derive the minimum variance, we define $\check{S} := \sum_{m=1}^k \frac{\theta_m^2}{n_m} \check{C}_m, \quad \check{\mathbf{u}} := \sum_{m=1}^{k-1} \frac{\theta_m^2}{n_m} \check{C}_m \mathbf{e}_m - \frac{\theta_k^2}{n_k} \check{C}_k \mathbf{1}$. The minimum variance is derived similarly to Theorem \ref{thm:min_var}.
	
	\begin{thm}[Minimum Variance with Symmetric Loss]\label{thm:min_var_sym} Assume $\check{S}$ is invertible. The minimum variance achievable by the linear risk estimator $\hat{R}_{lin}^{\boldsymbol{a}}$ under symmetric loss is given by
		\begin{equation}
			\min_{\boldsymbol{a}} \operatorname{Var}\left[\hat{R}_{lin}^{\boldsymbol{a}}\right] = \sum_{m=1}^{k} \frac{\theta_m^2}{n_m} (C_m)_{mm} - \check{\mathbf{u}}^T \check{S}^{-1} \check{\mathbf{u}}.
		\end{equation}
	\end{thm}

	\subsection{Comparison in Binary Classification with Symmetric Loss}
	We compare the proposed method, PNPU and PNNU learning in binary classification ($k=2$) under the symmetric loss assumption. We demonstrate that in this restricted setting, these methods are mathematically equivalent.
	
	\begin{thm}[Equivalence of risk estimators]\label{thm:equiv}
		Consider the binary classification setting ($k=2$) with a symmetric loss function. Let $S_{\text{PNPU}} := \{ R_{\text{PNPU}}^{\eta} \mid \eta \in \mathbb{R} \}$ and $S_{\text{PNNU}} := \{R_{\text{PNNU}}^{\eta} \mid \eta \in \mathbb{R} \}$ denote the sets of PNPU and PNNU risks. Then,
		\begin{equation}
			S_{lin} = S_{\text{PNPU}} = S_{\text{PNNU}}.
		\end{equation}
	\end{thm}
	
	Therefore, they achieve the same minimum variance, which implies that, PNU is already variance-optimal within the class of linear unbiased risk rewritings in this setting.
	
	\section{Relationship between semi-supervised learning and variance reduction}
	The analysis in the previous sections focused on the pointwise variance of the unbiased risk estimator for a fixed classifier $g$. The goal of semi-supervised learning is to identify the optimal predictor $\hat{g}$ using the empirical risk. In this section, we establish the theoretical connection between reducing the variance of the risk estimator and improving the generalization bound of the learned classifier.

	\subsection{Generalization bound with reduced variance}We consider $n_U\rightarrow\infty$ to analyze the generalization performance. Let $\hat{g} = \operatorname{argmin}_{g \in \mathcal{G}} \hat{R}_{lin}^{\boldsymbol{a}}(g)$ be the empirical risk minimizer using our proposed estimator, and let $g^* = \operatorname{argmin}_{g \in \mathcal{G}} R(g)$ be the true risk minimizer. 
	
	\begin{thm}[Generalization Bound for $\hat{R}_{lin}^{\boldsymbol{a}}$]\label{thm:gen_bound}
		Assume $n_U\rightarrow\infty$ and that the loss function $l(z, y)$ is $L$-Lipschitz continuous with respect to the prediction $z \in \mathbb{R}^k$, and bounded such that $0 \leq l(\cdot, \cdot) \leq c_l$.
		Assume $\max_i\left|\frac{a_i}{\theta_i}-1\right| \leq c_{\boldsymbol{a}}$.
		Let $\mathcal{N}(\mathcal{G}, \nu, \|\cdot\|_\infty)$ be the $\nu$-covering number of $\mathcal{G}$ with respect to the $L_\infty$ norm \citep{wain19high}.
		Define the maximum variance on the $\nu$-cover $C_\nu$ as $\sigma^2_{\max}(\boldsymbol{a}, \nu) := \max_{g' \in C_\nu} \operatorname{Var}[\hat{R}_{lin}^{\boldsymbol{a}}(g')]$.
		Then, for any $\nu > 0$, with probability at least $1-\delta$, the excess risk of $\hat{g}$ satisfies:
		\begin{equation*}
			\resizebox{0.95\columnwidth}{!}{$
				\begin{aligned}
					R(\hat{g}) - R(g^*) &\leq 4 L_{\boldsymbol{a}} \nu + \sqrt{8 \sigma_{\max}^2(\boldsymbol{a}, \nu) \ln\left(\frac{2\mathcal{N}(\mathcal{G}, \nu, \|\cdot\|_\infty)}{\delta}\right)} \\
					&\quad + \frac{4}{3}B_{\boldsymbol{a}}\max_{m\in[k]}\left(\frac{\theta_m}{n_m}\right) \ln\left(\frac{2\mathcal{N}(\mathcal{G}, \nu, \|\cdot\|_\infty)}{\delta}\right),
				\end{aligned}
				$}
		\end{equation*}
		where $B_{\boldsymbol{a}}:=k c_l\left(1+c_{\boldsymbol{a}}\right)$ and $L_{\boldsymbol{a}} := L(1 + k c_{\boldsymbol{a}})$.
	\end{thm}
	
	Let $n = \min_m(n_m)$. As $n \rightarrow \infty$, the maximum variance $\sigma_{\max}^2(\boldsymbol{a}, \nu)$ scales as $\mathcal{O}(1/n)$, which implies that the second term of the generalization bound is $\mathcal{O}(1/\sqrt{n})$. Furthermore, since the third term is proportional to $\max_{m\in[k]}\left(\frac{\theta_m}{n_m}\right)$, it scales as $\mathcal{O}(1/n)$. Considering only the dominant $\mathcal{O}(1/\sqrt{n})$ term and assuming $\nu$ is small, this theorem implies that the generalization error is small when the maximum variance $\sigma_{\max}^2(\boldsymbol{a}, \nu)$ is small. Thus, a good choice of $\boldsymbol{a}$ improves the SSL performance.
	
	\subsection{Practical implementation of SSL}\label{subsec:practicalssl}
	Based on the relationship established above, we propose two SSL methods to learn the classifier $\hat g$.
	
	\paragraph{1. Iterative optimization method} 
	\begin{algorithm}[tbp]
		\caption{Iterative Optimization SSL}
		\label{alg:iterative_opt}
		\begin{algorithmic}[1]
			\Require $\mathcal{X}_L = \cup_{i=1}^k \mathcal{X}_i$, $\mathcal{X}_U$, $\mathcal{X}_{val}$, update interval $m$, learning rate $\eta$, epochs $E$.
			\State Initialize classifier $\hat{g}$ and risk parameter $\boldsymbol{a}$.
			\For{$e = 1$ \textbf{to} $E$}
			\If{$e \pmod m == 0$}
			\State Estimate covariance $C_i$ using outputs of $\hat{g}$ on $\mathcal{X}_{val}$.
			\State Update $\boldsymbol{a} \leftarrow \boldsymbol{a}^*$ minimizing variance (Theorem~\ref{thm:min_var}) on $\mathcal{X}_{val}$.
			\EndIf
			\State Update $\hat{g} \leftarrow \hat{g} - \eta \nabla_{\hat{g}} \hat{R}_{lin}^{\boldsymbol{a}}(\hat{g})$ using $\mathcal{X}_L \cup \mathcal{X}_U$.
			\EndFor
			\State \Return $\hat{g}$
		\end{algorithmic}
	\end{algorithm}
	
	Since the optimal parameter $\boldsymbol{a}$ that minimizes the bound in Theorem \ref{thm:gen_bound} is unknown in practice, we propose empirical iterative optimization of $g$ and $\boldsymbol{a}$. The algorithm is described in Algorithm \ref{alg:iterative_opt}. It optimizes $\hat g$ with $\hat R_{lin}^{\boldsymbol{a}}$, calculating $\boldsymbol{a}$ every $m$ epochs to minimize the variance with the current $\hat g$ using validation data.
	
	\paragraph{2. Data-free method under equal covariance assumption} The first method requires a certain size of validation data to accurately estimate $\boldsymbol{a}^*$. The estimation becomes difficult when the number of classes $k$ is large. If we assume the equal covariance $\forall m\in[k], C_m=C$, the optimal $\boldsymbol{a}^*$ is calculated as $\boldsymbol{a}^*_i=\theta_i\left(1-\frac{w_i}{\sum_{m=1}^k w_m}\right)$ where $w_i:=\theta_i^2/n_i$. We use the linear risk with this $\boldsymbol{a}^*$ to learn $\hat g$. This does not require validation data to estimate $\boldsymbol{a}^*$.	
	
	\paragraph{Non-negative risk correction for complex models}
	Note that we can express the linear risk estimator as$$\hat R_{lin}^{\boldsymbol{a}}(g)=\sum^k_{j=1}\left(a_j\hat R_{jj}(g)+\left(1-\frac{a_j}{\theta_j}\right)\Delta_j(g)\right),$$where $\Delta_j(g):=\hat R_{Uj}(g)-\sum_{i\neq j}^k\theta_i\hat R_{ij}(g)$. Here, $\Delta_j(g)$ can be viewed as a risk-rewriting estimator for $\theta_j R_{jj}(g)$. As seen in PU and other risk-rewriting methods \citep{kiryo17pu,lu20mit,tang23multi}, when $\mathcal{G}$ consists of complex models such as deep neural networks, $\Delta_j(g)$ can become negative even when $R_{jj}(g)$ is non-negative, which leads to overfitting. To mitigate this issue, similarly to the approach by \citet{kiryo17pu}, we replace $\Delta_j(g)$ within $R_{lin}^{\boldsymbol{a}}(g)$ with a non-negative counterpart, $\Delta_j^{nn}(g):=\max\{0,\Delta_j(g)\}$, when training complex models.
	
	\section{Experiments}
	\subsection{Variance comparison of $R_{PN}$, $R_{PNU}$ and $R_{lin}^{\boldsymbol{a}^*}$}\label{subsec:varcomp}
	In this subsection, we empirically validate the variance reduction achieved by the proposed linear risk estimator $\hat{R}_{lin}^{\boldsymbol{a}}(g)$. We compare the variance of three unbiased risk estimators for a pre-trained classifier $g$: the standard supervised risk estimator $\hat{R}(g)$ (PN risk, which corresponds to $\hat{R}_{lin}^{\boldsymbol{a}}(g)$ where $\boldsymbol{a}_i=\theta_i$), the PNU risk estimator, and our proposed linear risk estimator with optimal parameter~$\boldsymbol{a}^*$. We present binary classification results ($k=2$) here; multiclass results are deferred to the Appendix.
	
	\paragraph{Datasets.}
	We use two datasets: a synthetic two-dimensional Gaussian dataset and Credit default dataset from the UCI Machine Learning Repository \citep{kelly23UCI}. For the Gaussian dataset, the class-conditional distributions are
	$p_1(x) = \mathcal{N}\!\bigl((1,\; 1)^{\!\top}, I_2\bigr)$ and  
	$p_2(x) = \mathcal{N}\!\bigl((0,\; 0)^{\!\top}, I_2\bigr)$.

	\paragraph{Model and training.}
	We train a logistic regression model $g(x) = w^{\!\top}\!x + b$ by minimizing the binary cross-entropy (BCE) loss with $n_1 = n_2 = 30$ labeled samples via stochastic gradient descent.
	
	\paragraph{Evaluation loss.}
	For the variance comparison, we use two loss functions: the BCE loss
	$l(g(x), y) = -y \log \sigma(g(x)) - (1-y) \log(1-\sigma(g(x)))$ as an asymmetric loss, and the 0-1 loss as a symmetric loss.
	
	\paragraph{Evaluation protocol.}
	With the trained classifier $g$ fixed, we evaluate the variance of each risk estimator by repeatedly sampling from held-out data. To compute the optimal parameters $\eta^*$ and $\boldsymbol{a}^*$, we estimate the covariance matrices $C_m$ for all $m \in [k]$ using the entire held-out data. For the PNU risk, we consider the set $S_{PNU}:=S_{PNPU}\cup S_{PNNU}$ as defined in Theorem~\ref{thm:equiv}.
	
	For each experimental condition (class prior $\theta_1$, unlabeled sample size $n_U$), we conduct $5{,}000$ independent trials to estimate the risk variances. We vary $\theta_1 \in \{0.3, 0.5, 0.7\}$ and $n_U \in \{50, 100, 200, 500, 1000\}$. The labeled sample sizes for risk estimation are fixed at $n_1 = n_2 = 30$ across all settings.

	\subsubsection{Results}
	\label{sec:results}
	
	Figure~\ref{fig:var_compare_all} visualizes the variance ratios $\mathrm{Var}(\hat{R}_{\mathrm{PNU}})/
	\mathrm{Var}(\hat{R})$ and $\mathrm{Var}(\hat{R}_{lin}^{\boldsymbol{a}^{*}})/
	\mathrm{Var}(\hat{R})$ as a function of
	$n_U$ for each class prior.
	
	\begin{figure}[tbp]
		\centering
		\begin{subfigure}[b]{0.48\columnwidth}
			\centering
			\includegraphics[width=\textwidth]{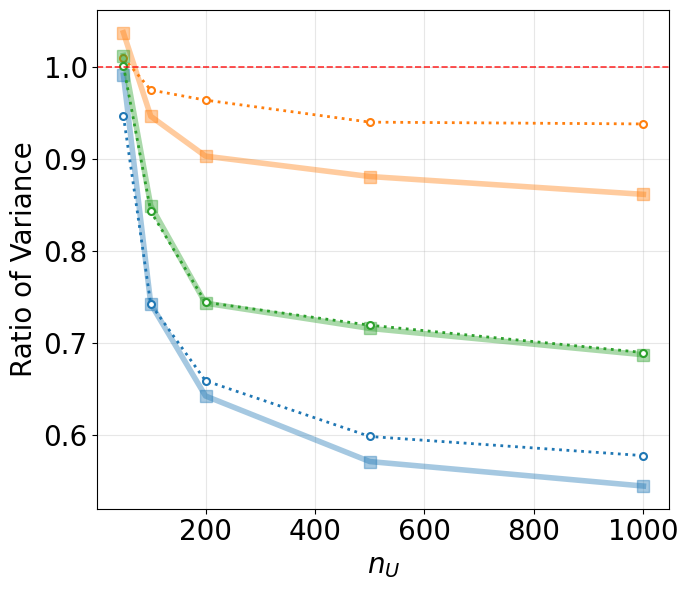}
			\vspace{-1.8em}
			\caption{Gaussian (BCE loss)}
			\label{fig:gaussian_bce}
		\end{subfigure}
		\hfill
		\begin{subfigure}[b]{0.48\columnwidth}
			\centering
			\includegraphics[width=\textwidth]{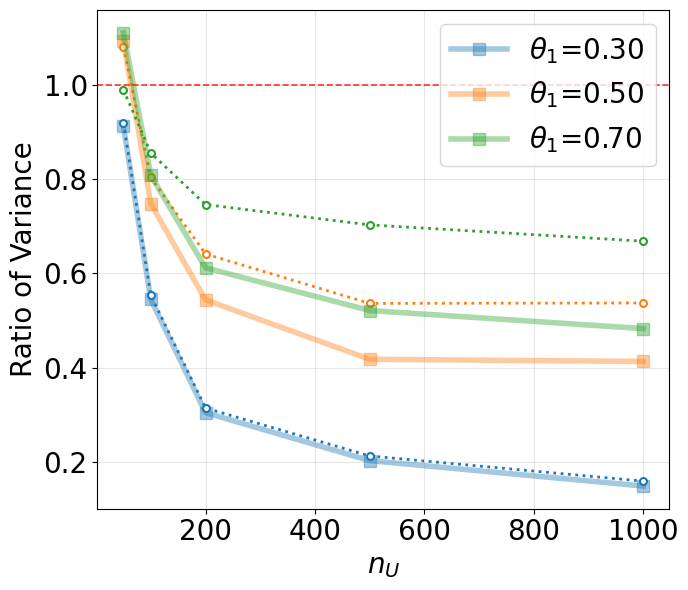}
			\vspace{-1.8em}
			\caption{Credit (BCE loss)}
			\label{fig:credit_bce}
		\end{subfigure}
		
		\vspace{0.2cm} 
		\begin{subfigure}[b]{0.48\columnwidth}
			\centering
			\includegraphics[width=\textwidth]{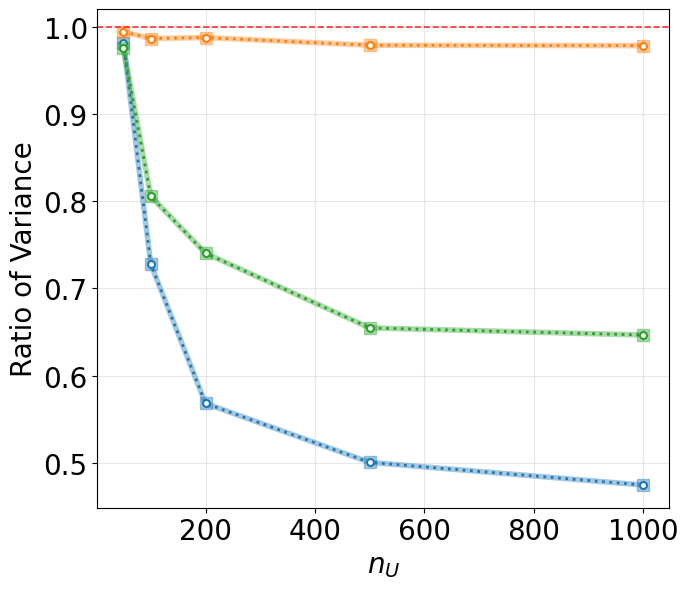}
			\vspace{-1.8em}
			\caption{Gaussian (0-1 loss)}
			\label{fig:gaussian_01}
		\end{subfigure}
		\hfill
		\begin{subfigure}[b]{0.48\columnwidth}
			\centering
			\includegraphics[width=\textwidth]{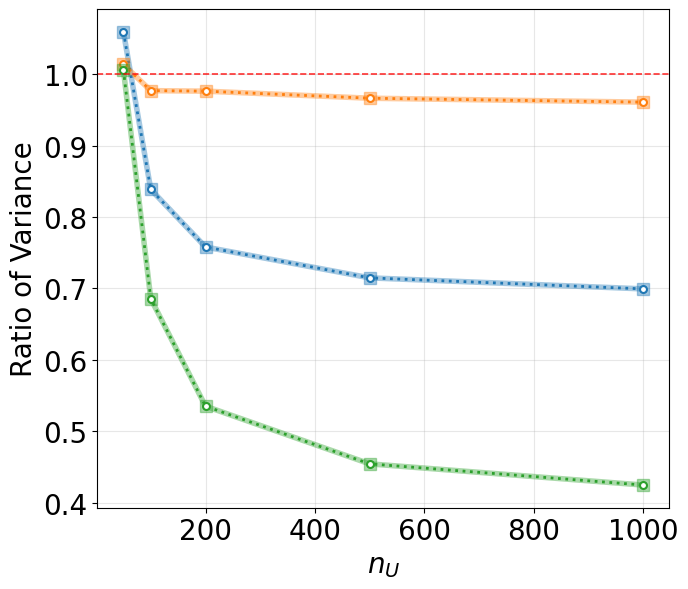}
			\vspace{-1.8em}
			\caption{Credit (0-1 loss)}
			\label{fig:credit_01}
		\end{subfigure}
		
		\caption{Variance ratio (each method / PN) as a function of the unlabeled sample size $n_U$ for $\theta_1 \in \{0.3,\; 0.5,\; 0.7\}$. Solid lines denote the proposed linear risk estimator ($\hat{R}_{lin}^{\boldsymbol{a}^*}$); dotted lines denote the PNU risk estimator. The red dashed line indicates the supervised baseline ($\mathrm{ratio} = 1$). Both methods benefit from increasing $n_U$, and the proposed method consistently achieves equal or lower variance than PNU.}
		\label{fig:var_compare_all}
	\end{figure}

	\paragraph{Effect of unlabeled sample size.}
	As the number of unlabeled samples $n_U$ increases, both the PNU and the proposed linear risk estimator achieve progressively lower variance ratios, confirming that unlabeled data effectively reduces the variance of risk estimation.
	
	\paragraph{Effect of class prior.}
	When the class prior is imbalanced ($\theta_1 = 0.3$ or $\theta_1 = 0.7$),
	both methods achieve substantial variance reduction. These results are consistent with Theorem~\ref{thm:varred}.
	In the balanced case ($\theta_1 = 0.5$) with the asymmetric loss (BCE), the PNU risk shows only marginal
	improvement (e.g., a ratio of $0.953$ at $n_U = 1{,}000$), whereas the proposed linear
	risk estimator still achieves a meaningful reduction (a ratio of $0.883$) , as predicted by Theorem~\ref{thm:variance_dominance}. 
	
	\paragraph{Effect of loss function symmetry.}
	When the loss function is symmetric, the variance ratios of the proposed estimator and the PNU estimator are nearly identical. This is expected, since their optimal variances coincide under symmetric loss, as established in Theorem~\ref{thm:equiv}.

	\subsection{Comparison with existing semi-supervised learning methods}\label{subsec:sslexp}
	
	\begin{table}[t]
		\centering
		\caption{Binary classification results (test accuracy \%, mean$\pm$std over 30 seeds). The best and second-best methods per row are \textbf{bolded} and \underline{underlined}, respectively.}
		\label{tab:binary}
		\resizebox{\columnwidth}{!}{
			\begin{tabular}{ll @{\hspace{4pt}}c@{\hspace{4pt}}c@{\hspace{4pt}}c@{\hspace{4pt}}c@{\hspace{4pt}}c}
				\toprule
				Dataset & $(n_1, n_2)$ & Sup. & PNU & Ours (Iter) & PL & VAT \\
				\midrule
				\multirow{4}{*}{Adult}
				& $(15, 45)$   & $\underline{79.7{\pm}2.4}$ & $79.1{\pm}2.4$ & $\mathbf{80.3{\pm}1.6}$ & $78.9{\pm}1.8$ & $78.4{\pm}2.7$ \\
				& $(30, 30)$   & $\underline{80.4{\pm}2.1}$ & $79.7{\pm}2.5$ & $\mathbf{80.8{\pm}1.5}$ & $78.2{\pm}2.1$ & $78.2{\pm}2.0$ \\
				& $(50, 150)$  & $\underline{81.5{\pm}2.0}$ & $80.5{\pm}2.4$ & $\mathbf{82.2{\pm}1.0}$ & $79.4{\pm}1.9$ & $79.6{\pm}1.9$ \\
				& $(100, 100)$ & $\underline{81.5{\pm}1.8}$ & $81.3{\pm}2.1$ & $\mathbf{82.4{\pm}1.0}$ & $78.9{\pm}2.0$ & $78.7{\pm}2.2$ \\
				\midrule
				\multirow{4}{*}{Banknote}
				& $(15, 45)$   & $95.6{\pm}3.2$ & $\underline{96.9{\pm}2.0}$ & $\mathbf{97.3{\pm}1.5}$ & $94.3{\pm}4.4$ & $95.7{\pm}3.8$ \\
				& $(30, 30)$   & $96.8{\pm}2.4$ & $\underline{97.2{\pm}1.3}$ & $\mathbf{97.8{\pm}1.1}$ & $95.7{\pm}3.6$ & $97.1{\pm}1.6$ \\
				& $(50, 150)$  & $\underline{98.4{\pm}1.0}$ & $97.3{\pm}1.0$ & $\mathbf{98.6{\pm}0.9}$ & $97.2{\pm}1.8$ & $97.8{\pm}1.0$ \\
				& $(100, 100)$ & $\underline{98.6{\pm}1.0}$ & $98.0{\pm}0.8$ & $\mathbf{98.8{\pm}1.0}$ & $97.6{\pm}1.1$ & $98.0{\pm}0.9$ \\
				\midrule
				\multirow{4}{*}{\shortstack[l]{Breast\\Cancer}}
				& $(15, 45)$   & $93.6{\pm}2.9$ & $\underline{93.9{\pm}2.8}$ & $\mathbf{94.5{\pm}2.3}$ & $93.4{\pm}2.5$ & $93.8{\pm}2.9$ \\
				& $(30, 30)$   & $\underline{93.9{\pm}2.2}$ & $\underline{93.9{\pm}2.5}$ & $\mathbf{94.4{\pm}1.8}$ & $93.1{\pm}2.7$ & $93.5{\pm}3.0$ \\
				& $(50, 150)$  & $\underline{95.4{\pm}1.8}$ & $94.9{\pm}2.3$ & $\mathbf{95.5{\pm}1.7}$ & $94.8{\pm}2.1$ & $95.1{\pm}2.1$ \\
				& $(100, 100)$ & $\mathbf{95.6{\pm}1.6}$ & $94.3{\pm}2.6$ & $\mathbf{95.6{\pm}1.8}$ & $\underline{94.4{\pm}2.0}$ & $94.3{\pm}2.3$ \\
				\midrule
				\multirow{4}{*}{Credit}
				& $(15, 45)$   & $\mathbf{79.0{\pm}1.1}$ & $\underline{78.6{\pm}1.3}$ & $78.2{\pm}2.0$ & $72.8{\pm}4.6$ & $71.7{\pm}4.5$ \\
				& $(30, 30)$   & $78.6{\pm}1.6$ & $\mathbf{79.2{\pm}1.0}$ & $\underline{78.9{\pm}1.2}$ & $66.3{\pm}6.5$ & $67.8{\pm}5.6$ \\
				& $(50, 150)$  & $\mathbf{79.5{\pm}1.2}$ & $\mathbf{79.5{\pm}0.9}$ & $\underline{78.9{\pm}1.3}$ & $72.1{\pm}4.5$ & $74.1{\pm}3.9$ \\
				& $(100, 100)$ & $\mathbf{79.8{\pm}1.1}$ & $\underline{79.7{\pm}0.8}$ & $79.3{\pm}1.1$ & $70.4{\pm}5.3$ & $72.5{\pm}4.6$ \\
				\midrule
				\multirow{4}{*}{\shortstack[l]{MNIST\\4v9}}
				& $(15, 45)$   & $92.0{\pm}3.9$ & $\underline{94.8{\pm}1.5}$ & $94.7{\pm}1.9$ & $\mathbf{95.2{\pm}3.0}$ & $93.0{\pm}3.3$ \\
				& $(30, 30)$   & $94.1{\pm}2.5$ & $95.5{\pm}1.0$ & $\underline{95.6{\pm}1.6}$ & $\mathbf{96.5{\pm}1.7}$ & $95.2{\pm}1.3$ \\
				& $(50, 150)$  & $96.1{\pm}1.3$ & $96.9{\pm}1.1$ & $\mathbf{97.3{\pm}0.9}$ & $\underline{97.2{\pm}0.8}$ & $\underline{97.2{\pm}0.6}$ \\
				& $(100, 100)$ & $96.6{\pm}1.7$ & $97.3{\pm}0.9$ & $\mathbf{97.7{\pm}0.7}$ & $97.4{\pm}1.4$ & $\underline{97.6{\pm}0.5}$ \\
				\midrule
				\multirow{4}{*}{\shortstack[l]{CIFAR\\CvD}}
				& $(15, 45)$   & $54.3{\pm}2.8$ & $\underline{55.0{\pm}2.4}$ & $\mathbf{55.4{\pm}1.6}$ & $53.8{\pm}2.5$ & $54.6{\pm}2.5$ \\
				& $(30, 30)$   & $55.2{\pm}2.6$ & $\mathbf{56.0{\pm}2.1}$ & $\underline{55.8{\pm}2.0}$ & $55.7{\pm}2.6$ & $55.4{\pm}2.4$ \\
				& $(50, 150)$  & $57.7{\pm}2.5$ & $\mathbf{58.3{\pm}2.1}$ & $\underline{58.2{\pm}2.1}$ & $57.1{\pm}2.7$ & $57.7{\pm}2.2$ \\
				& $(100, 100)$ & $59.2{\pm}1.9$ & $\mathbf{59.6{\pm}2.1}$ & $\mathbf{59.6{\pm}1.7}$ & $59.2{\pm}1.5$ & $\underline{59.4{\pm}1.4}$ \\
				\bottomrule
		\end{tabular}}
	\end{table}
	
	\begin{table}[t]
		\centering
		\caption{Multiclass classification results (test accuracy \%, mean$\pm$std over 30 seeds). The best and second-best methods per row are \textbf{bolded} and \underline{underlined}, respectively.}
		\label{tab:multi_tabular}
		\resizebox{\columnwidth}{!}{
				\begin{tabular}{ll @{\hspace{4pt}}c@{\hspace{4pt}}c@{\hspace{4pt}}c@{\hspace{4pt}}c@{\hspace{4pt}}c}
					\toprule
					Dataset & Regime ($n$) & Sup. & Ours (Iter) & Ours (EC) & PL & VAT \\
					\midrule
					\multirow{7}{*}{Covertype}
					& Bal.\ (70)   & $\underline{52.5{\pm}3.1}$ & $52.0{\pm}2.9$ & $\mathbf{53.1{\pm}3.4}$ & $42.1{\pm}4.9$ & $42.4{\pm}4.6$ \\
					& Bal.\ (140)  & $\mathbf{56.5{\pm}2.5}$ & $55.0{\pm}2.6$ & $\underline{56.1{\pm}2.7}$ & $46.7{\pm}3.9$ & $46.6{\pm}3.1$ \\
					& Bal.\ (350)  & $\mathbf{61.6{\pm}1.8}$ & $\underline{61.2{\pm}2.1}$ & $\underline{61.2{\pm}1.9}$ & $52.9{\pm}2.4$ & $51.8{\pm}2.5$ \\
					& Mild (140)   & $56.9{\pm}2.8$ & $\underline{57.4{\pm}2.1}$ & $\mathbf{57.7{\pm}1.8}$ & $49.6{\pm}2.9$ & $49.8{\pm}2.2$ \\
					& Mild (350)   & $62.0{\pm}2.0$ & $\mathbf{62.7{\pm}1.4}$ & $\underline{62.3{\pm}1.7}$ & $54.3{\pm}2.3$ & $54.7{\pm}2.1$ \\
					& Sev.\ (140)  & $\mathbf{58.8{\pm}2.3}$ & $\underline{58.6{\pm}2.5}$ & $\mathbf{58.8{\pm}2.0}$ & $53.7{\pm}2.8$ & $54.4{\pm}2.3$ \\
					& Sev.\ (350)  & $63.2{\pm}1.6$ & $\underline{63.3{\pm}1.7}$ & $\mathbf{63.7{\pm}1.5}$ & $58.4{\pm}2.5$ & $58.2{\pm}2.2$ \\
					\midrule
					\multirow{7}{*}{Shuttle}
					& Bal.\ (70)   & $94.3{\pm}3.1$ & $\mathbf{96.7{\pm}1.2}$ & $\underline{96.2{\pm}1.3}$ & $89.9{\pm}4.6$ & $95.5{\pm}2.4$ \\
					& Bal.\ (140)  & $94.9{\pm}3.0$ & $\mathbf{97.3{\pm}1.0}$ & $\underline{97.1{\pm}1.0}$ & $93.8{\pm}3.8$ & $96.7{\pm}0.7$ \\
					& Bal.\ (350)  & $97.2{\pm}1.1$ & $\underline{97.3{\pm}1.0}$ & $\mathbf{97.4{\pm}0.7}$ & $96.9{\pm}1.3$ & $96.7{\pm}1.2$ \\
					& Mild (140)   & $\mathbf{96.9{\pm}1.1}$ & $\mathbf{96.9{\pm}1.0}$ & $\mathbf{96.9{\pm}0.9}$ & $96.2{\pm}2.0$ & $\underline{96.5{\pm}1.3}$ \\
					& Mild (350)   & $97.2{\pm}1.0$ & $\mathbf{97.6{\pm}1.1}$ & $\underline{97.3{\pm}0.7}$ & $97.0{\pm}1.0$ & $96.7{\pm}0.7$ \\
					& Sev.\ (140)  & $97.0{\pm}1.2$ & $\mathbf{97.2{\pm}0.9}$ & $\underline{97.1{\pm}1.7}$ & $96.1{\pm}1.6$ & $96.0{\pm}1.6$ \\
					& Sev.\ (350)  & $97.3{\pm}0.8$ & $\mathbf{97.6{\pm}0.8}$ & $\underline{97.4{\pm}1.0}$ & $97.3{\pm}0.8$ & $96.8{\pm}1.0$ \\
					\midrule
					\multirow{7}{*}{DryBeans}
					& Bal.\ (70)   & $86.4{\pm}1.7$ & $87.6{\pm}1.6$ & $86.6{\pm}1.6$ & $\underline{88.1{\pm}2.0}$ & $\mathbf{88.5{\pm}1.3}$ \\
					& Bal.\ (140)  & $88.8{\pm}1.6$ & $\underline{89.9{\pm}0.9}$ & $89.2{\pm}1.0$ & $\mathbf{90.0{\pm}1.1}$ & $\underline{89.9{\pm}1.4}$ \\
					& Bal.\ (350)  & $90.5{\pm}0.7$ & $\mathbf{91.3{\pm}0.6}$ & $90.4{\pm}0.9$ & $\underline{90.8{\pm}0.9}$ & $\underline{90.8{\pm}1.0}$ \\
					& Mild (140)   & $88.8{\pm}1.6$ & $\underline{89.1{\pm}1.7}$ & $88.9{\pm}1.6$ & $88.7{\pm}1.8$ & $\mathbf{89.5{\pm}1.5}$ \\
					& Mild (350)   & $90.4{\pm}1.2$ & $\underline{90.7{\pm}0.9}$ & $\underline{90.7{\pm}0.7}$ & $90.6{\pm}1.0$ & $\mathbf{90.8{\pm}0.8}$ \\
					& Sev.\ (140)  & $84.2{\pm}2.5$ & $84.9{\pm}2.6$ & $\underline{85.1{\pm}2.2}$ & $83.4{\pm}2.9$ & $\mathbf{85.3{\pm}2.2}$ \\
					& Sev.\ (350)  & $88.1{\pm}1.7$ & $\underline{89.1{\pm}1.5}$ & $88.7{\pm}1.4$ & $88.4{\pm}2.0$ & $\mathbf{89.2{\pm}1.3}$ \\
					\midrule
					\multirow{7}{*}{MNIST}
					& Bal.\ (200)  & $86.9{\pm}1.0$ & $86.8{\pm}1.7$ & $\underline{87.3{\pm}1.3}$ & $\mathbf{89.7{\pm}1.6}$ & $87.1{\pm}1.4$ \\
					& Bal.\ (500)  & $91.9{\pm}1.0$ & $91.8{\pm}1.4$ & $92.2{\pm}0.8$ & $\mathbf{92.9{\pm}0.9}$ & $\underline{92.3{\pm}0.7}$ \\
					& Bal.\ (1500) & $\underline{95.7{\pm}0.6}$ & $95.4{\pm}0.6$ & $95.4{\pm}0.6$ & $\mathbf{95.9{\pm}0.6}$ & $\underline{95.7{\pm}0.5}$ \\
					& Mild (500)   & $89.9{\pm}1.0$ & $89.8{\pm}1.4$ & $\underline{90.4{\pm}1.2}$ & $\mathbf{91.7{\pm}1.5}$ & $90.1{\pm}1.3$ \\
					& Mild (1500)  & $\underline{94.8{\pm}0.6}$ & $94.3{\pm}0.8$ & $94.4{\pm}0.7$ & $\mathbf{95.1{\pm}0.6}$ & $\underline{94.8{\pm}0.6}$ \\
					& Sev.\ (500)  & $77.3{\pm}2.5$ & $\mathbf{81.7{\pm}3.0}$ & $\underline{80.6{\pm}2.4}$ & $79.6{\pm}2.8$ & $77.1{\pm}2.6$ \\
					& Sev.\ (1500) & $88.5{\pm}1.5$ & $\underline{89.5{\pm}1.7}$ & $89.2{\pm}1.4$ & $\mathbf{90.7{\pm}2.2}$ & $88.4{\pm}1.7$ \\
					\bottomrule
			\end{tabular}}
		\end{table}

		We compare the proposed SSL methods against four baselines on binary and multiclass classification tasks: (1)~\textbf{Sup.}: standard supervised learning (linear risk with $\boldsymbol{a}_i:=\theta_i$); (2)~\textbf{PNU}: the PNU risk estimator of \citet{sakai17semi} with the equal variance assumption (binary only); (3)~\textbf{Pseudo Label (PL)}: self-training with pseudo-labels assigned by confidence thresholding, based on \citet{yaro95uns}; and (4)~\textbf{VAT}: virtual adversarial training \citep{miyato19vat}, which is a representative consistency regularization method. We denote our iterative method for optimizing $g$ and $\boldsymbol{a}$ in this subsection as \textbf{Ours (Iter)}. For multiclass tasks, we additionally report \textbf{Ours~(EC)}, the data-free method under the equal covariance assumption. We assume class priors $\theta_i$ are known. In practice, we can estimate the class priors using methods such as those proposed by \citet{kawakubo16comp,ram16kernel,moreo25kernel,lipton18black}. We use the cross-entropy loss for all methods. Detailed settings are provided in the Appendix.
		
		\paragraph{Datasets.}
		We use six binary and four multiclass classification tasks.
		\emph{Binary tabular}: Adult, Banknote, Breast Cancer, and Credit Default from the UCI Machine Learning Repository \citep{kelly23UCI}.
		\emph{Binary image}: MNIST digits 4 vs.\ 9 (MNIST~4v9) and CIFAR-10 cat vs.\ dog (CIFAR~CvD).
		\emph{Multiclass}: Covertype, Shuttle, and Dry Beans (7 classes), and MNIST (10 classes).
		For binary tasks, labeled samples are specified as $(n_1, n_2) \in \{(15,45), (30,30), (50,150), (100,100)\}$ with $n_U = 300$ (tabular) or $n_U = 5{,}000$ (image).
		For multiclass tasks, we test three regimes of labeled sample sizes---\emph{balanced} (equal samples per class), \emph{mild imbalance}, and \emph{severe imbalance}---with $n_U = 5{,}000$. Please refer to the Appendix for further details. The validation set consists of $60$ samples for all settings, except for the multiclass image task, which uses $200$ samples.
		
		\paragraph{Models and training.}
		We use multi-layer perceptrons (MLPs) for tabular datasets and convolutional neural networks (CNNs) for image datasets. All models are trained using the Adam optimizer. Since these models are complex, we apply non-negative risk correction in Section  \ref{subsec:practicalssl} to our methods and PNU learning.
		
		\subsubsection{Results}
		\paragraph{Binary classification results.}
		Table~\ref{tab:binary} summarizes the binary classification results. Across tabular datasets, risk-rewriting methods (Ours and PNU) consistently outperform Pseudo Label and VAT, often by a large margin. This is expected, as tabular datasets typically do not satisfy the distributional assumptions relied upon by those methods. In almost all settings, our method ranks first or second.
		
		On MNIST~4v9, Pseudo Label benefits from the well-separated cluster structure of the MNIST digits, achieving the best accuracy with very few labels $(n_1{=}15)$. However, on CIFAR-10 cat vs.\ dog, a more challenging task, it fails to yield comparable performance.
		
		\paragraph{Multiclass classification results.}
		Table~\ref{tab:multi_tabular} presents the multiclass results for the tabular datasets and MNIST. On Covertype and Shuttle, our methods \textbf{Ours (Iter)} and \textbf{Ours (EC)} consistently rank first or second, while PL and VAT sometimes underperform compared to standard supervised learning. 
		
		For Dry Beans and MNIST, although PL and VAT show strong performance, our methods achieve better or comparable accuracy to standard supervised learning. In summary, our methods are more robust than PL and VAT, which rely on distributional assumptions.
		
		It is worth highlighting that the data-free \textbf{Ours (EC)} method demonstrates highly competitive accuracy, matching or closely following the iterative optimization approach without requiring a validation set for parameter tuning. 
		
		\section{Conclusions}\label{sec:conclusion}	
		We proposed a generalized framework for distribution-free semi-supervised learning based on risk rewriting. By formulating the set of linear combinations of component risks $S_{lin}$, our framework subsumes existing methods such as PNU learning and naturally extends to multiclass classification.  Experiments confirmed that our methods consistently match or outperform existing approaches. 
		
		PNU learning has had a major impact on subsequent research in SSL and weakly supervised learning\citep{ishida17comple, shimada21pair, hien24anomaly, sakai18auc, hamm20pushift, wang25multi, kato20delayed, tsuchiya21ord, hayashi18mat}. These studies attempt to reduce the variance of risks by simply mixing rewritten risks similarly to PNU learning. By applying our framework, these methods can be formulated using the linear combination set to construct risk estimators with smaller variances, thereby improving learning efficiency. We leave the application of our generalized framework to these broad problem settings for future work.
%
%
		
		\begin{acknowledgements} 
This work was supported in part by JSPS KAKENHI Grant Numbers JP24K14849 and JP26H02491.
		\end{acknowledgements}
		
		\bibliography{ref}
		\newpage
		\onecolumn
		
		\input{appendix}

\end{document}

%% file: appendix.tex
\title{Appendix}
\maketitle

\appendix
\section{Omitted Proofs}
\begin{proof}[Proof of Theorem \ref{thm:rewrite_slin}] We can write $R_{lin}^{\{a_{ij}\}, \{b_j\}}(g)$ as 
	$$\begin{aligned}
		R_{lin}^{\{a_{ij}\}, \{b_j\}}(g) &= \sum_{i, j} a_{i j} R_{i j}(g) + \sum_{j} b_j R_{Uj}(g) \\
		&= \sum_{i, j} a_{i j} R_{i j}(g) + \sum_{j} b_j \left( \sum_{i} \theta_i R_{ij}(g) \right) \\
		&= \sum_{i, j} a_{i j} R_{i j}(g) + \sum_{i, j} \theta_i b_j R_{i j}(g) \\
		&= \sum_{i, j} (a_{i j} + \theta_i b_j) R_{i j}(g).
	\end{aligned}$$
	
	In addition, 
	$$R(g) = \sum_{i=1}^k \theta_i R_{ii}(g) = \sum_{i, j} \delta_{ij} \theta_i R_{ij}(g).$$
	
	Considering $\forall g\in\mathcal{G}, R_{lin}(g) = R(g)$,
	$$\sum_{i, j} \left( a_{i j} + \theta_i b_j - \delta_{ij} \theta_i \right) R_{i j}(g) = 0 \quad (\forall g \in \mathcal{G}).$$
	
	Then, under Assumption \ref{asmp:linindep}, $\forall i, \theta_i = a_{ii} + \theta_i b_i$ and $\forall i \neq j, a_{ij} + \theta_i b_j = 0$. Substituting $b_i=1-\frac{a_{ii}}{\theta_i}$ and $a_{ij}=-\theta_ib_j=\frac{\theta_i}{\theta_j}a_{jj}-\theta_i$ into $R_{lin}^{\{a_{ij}\}, \{b_j\}}(g)$, we derive the objective set.
\end{proof}

\begin{proof}[Proof of Theorem \ref{thm:min_var}]
	The minimum value of the convex quadratic function $f(\boldsymbol{a})$ is given by the algebraic identity:
	\begin{equation}
		f(\boldsymbol{a}^*) = c - \mathbf{b}^T A^{-1} \mathbf{b}
		\label{eq:quad_min}
	\end{equation}
	
	\textbf{Step 1: Simplifying the Term $\mathbf{b}^T A^{-1} \mathbf{b}$}
	First, we express $A$ and $\mathbf{b}$ in terms of $S$ and $Q$.
	$$A = \sum_{m=1}^k \frac{\theta_m^2}{n_m} Q C_m Q = Q \left(\sum_{m=1}^k \frac{\theta_m^2}{n_m} C_m\right) Q = Q S Q$$
	$$\mathbf{b} = \sum_{m=1}^k \frac{\theta_m^2}{n_m} Q C_m \mathbf{d}_m = Q \sum_{m=1}^k \frac{\theta_m^2}{n_m} C_m \mathbf{d}_m$$
	Let $\mathbf{v} := \sum_{m=1}^k \frac{\theta_m^2}{n_m} C_m \mathbf{d}_m$. Then $\mathbf{b} = Q \mathbf{v}$.
	Substituting these into the quadratic term in Eq. \eqref{eq:quad_min}:
	$$
	\begin{aligned}
		\mathbf{b}^T A^{-1} \mathbf{b} &= (Q \mathbf{v})^T (Q S Q)^{-1} (Q \mathbf{v}) \\
		&= \mathbf{v}^T Q (Q^{-1} S^{-1} Q^{-1}) Q \mathbf{v} \\
		&= \mathbf{v}^T S^{-1} \mathbf{v}
	\end{aligned}
	$$
	Next, we simplify $\mathbf{v}$ using $\mathbf{d}_m = \mathbf{1} - \mathbf{e}_m$:
	$$
	\begin{aligned}
		\mathbf{v} &= \sum_{m=1}^k \frac{\theta_m^2}{n_m} C_m (\mathbf{1} - \mathbf{e}_m) \\
		&= \left(\sum_{m=1}^k \frac{\theta_m^2}{n_m} C_m\right)\mathbf{1} - \sum_{m=1}^k \frac{\theta_m^2}{n_m} C_m \mathbf{e}_m \\
		&= S\mathbf{1} - \mathbf{u}
	\end{aligned}
	$$
	Thus, the reduction term is $\mathbf{v}^T S^{-1} \mathbf{v} = (S\mathbf{1}-\mathbf{u})^T S^{-1} (S\mathbf{1}-\mathbf{u})$.
	
	\textbf{Step 2: Expansion of $\mathbf{v}^T S^{-1} \mathbf{v}$}
	Expanding the term derived above:
	$$
	\begin{aligned}
		\mathbf{v}^T S^{-1} \mathbf{v} &= (\mathbf{1}^TS-\mathbf{u}^T) S^{-1} (S\mathbf{1}-\mathbf{u}) \\
		&= \mathbf{1}^T S S^{-1} S\mathbf{1} - \mathbf{u}^T S^{-1} S\mathbf{1} - \mathbf{1}^T S S^{-1} \mathbf{u}+\mathbf{u}^T S^{-1} \mathbf{u}  \\
		&= \mathbf{1}^T S \mathbf{1} - 2\mathbf{1}^T \mathbf{u}+\mathbf{u}^T S^{-1} \mathbf{u}.
	\end{aligned}
	$$
	
	\textbf{Step 3: Expansion of the Constant $c$}
	We expand $c = \sum_{m=1}^k \frac{\theta_m^2}{n_m} \mathbf{d}_m^T C_m \mathbf{d}_m$:
	$$
	\begin{aligned}
		c &= \sum_{m=1}^k \frac{\theta_m^2}{n_m} (\mathbf{1} - \mathbf{e}_m)^T C_m (\mathbf{1} - \mathbf{e}_m) \\
		&= \sum_{m=1}^k \frac{\theta_m^2}{n_m} \left( \mathbf{1}^T C_m \mathbf{1} - 2\mathbf{1}^T C_m \mathbf{e}_m + \mathbf{e}_m^T C_m \mathbf{e}_m \right)\\
		&= \mathbf{1}^T \left(\sum_{m=1}^k \frac{\theta_m^2}{n_m} C_m\right) \mathbf{1} - 2\mathbf{1}^T \left(\sum_{m=1}^k \frac{\theta_m^2}{n_m} C_m \mathbf{e}_m\right) + \sum_{m=1}^k \frac{\theta_m^2}{n_m} \mathbf{e}_m^T C_m \mathbf{e}_m \\
		&= \mathbf{1}^T S \mathbf{1} - 2\mathbf{1}^T \mathbf{u} + \sum_{m=1}^k \frac{\theta_m^2}{n_m} (C_m)_{mm}
	\end{aligned}
	$$
	
	\textbf{Final Calculation}
	Substituting the expanded forms of $c$ and $\mathbf{v}^T S^{-1} \mathbf{v}$ back into Eq. \eqref{eq:quad_min}:
	\begin{align*}
		f(\boldsymbol{a}^*) &= c - \mathbf{v}^T S^{-1} \mathbf{v} \\
		&= \left( \mathbf{1}^T S \mathbf{1} - 2\mathbf{1}^T \mathbf{u} + \sum_{m=1}^k \frac{\theta_m^2}{n_m} (C_m)_{mm} \right) - \left( \mathbf{1}^T S \mathbf{1} - 2\mathbf{1}^T \mathbf{u}+\mathbf{u}^T S^{-1} \mathbf{u}\right)\\
		&=\sum_{m=1}^k \frac{\theta_m^2}{n_m} (C_m)_{mm} - \mathbf{u}^T S^{-1} \mathbf{u}
	\end{align*}
\end{proof}

\begin{proof}[Proof of Theorem \ref{thm:varred}]
	Let us define the weight for class $m$ as $w_m$ and the total weight $W$:$$w_m := \frac{\theta_m^2}{n_m}, \quad W := \sum_{m=1}^k w_m = \sum_{m=1}^k \frac{\theta_m^2}{n_m}.$$ We also define a weight vector $\mathbf{w} \in \mathbb{R}^k$ as $\mathbf{w} = [w_1, w_2, \dots, w_k]^T$. Under the assumption that $C_m = C$ for all $m$, the total weighted covariance matrix $S$  simplifies to:
	$$S = \sum_{m=1}^k w_m C_m = \left(\sum_{m=1}^k w_m\right) C = W C.$$ Similarly, the weighted covariance vector $\mathbf{u}$  simplifies to
	$$\mathbf{u} = \sum_{m=1}^k w_m C_m \mathbf{e}_m = C \sum_{m=1}^k w_m \mathbf{e}_m = C \mathbf{w}.$$ 
	Then we can write 
	\begin{align*}
		\mathbf{u}^T S^{-1} \mathbf{u} &= (C \mathbf{w})^T (W C)^{-1} (C \mathbf{w})= \mathbf{w}^T C \frac{1}{W} C^{-1} C \mathbf{w}\\ &= \frac{1}{W} \mathbf{w}^T C \mathbf{w}.
	\end{align*}
	Under the assumption that $C$ has diagonal elements $\rho_1$ and off-diagonal elements $\rho_2$, we have$$C = (\rho_1 - \rho_2)I + \rho_2 \mathbf{1}\mathbf{1}^T.$$
	
	Evaluating the quadratic form $\mathbf{w}^T C \mathbf{w}$ yields\begin{align*}\mathbf{w}^T C \mathbf{w} &= \mathbf{w}^T \left( (\rho_1 - \rho_2)I + \rho_2 \mathbf{1}\mathbf{1}^T \right) \mathbf{w}= (\rho_1 - \rho_2) |\mathbf{w}|^2 + \rho_2 (\mathbf{w}^T \mathbf{1})^2 \\&= (\rho_1 - \rho_2) \sum_{m=1}^k w_m^2 + \rho_2 W^2,\end{align*}
	This completes the proof of the equality. To derive the upper bound, we note that $\sum_{m=1}^k w_m^2 \leq (\max_{m} w_m) \sum_{m=1}^k w_m = W \max_{m} w_m$.
	Since $\rho_1 > \rho_2$, we have
	\begin{align*}
		\mathbf{u}^\top S^{-1} \mathbf{u} &\leq (\rho_1 - \rho_2) \frac{W \max_{m} w_m}{W} + \rho_2 W \\
		&= \rho_1 \max_{m} w_m - \rho_2 \max_{m} w_m + \rho_2 W \\
		&= \rho_1 \max_{m \in [k]} w_m + \rho_2 \left( W - \max_{m \in [k]} w_m \right).
	\end{align*}
\end{proof}

\begin{proof}[Proof of Theorem \ref{thm:variance_dominance}]
	\textbf{Dominance inequality \eqref{eq:var_com}:}
	Recall that 
	\begin{align*}
		R_{\text{PNPU}}^{\eta}(g) &= (1-\eta)(\theta_1 R_{11} + \theta_2 R_{22}) + \eta(\theta_1 R_{11} - \theta_1 R_{12} + R_{U2}) \\
		&= \theta_1 R_{11} + (1-\eta)\theta_2 R_{22} -\eta\theta_1R_{12}+\eta R_{U2}.
	\end{align*}
	This risk is equivalent to $R_{lin}^{\boldsymbol{a}}(g)$ where $(a_1,a_2)=(\theta_1,\theta_2(1-\eta))$. Similarly, 
	\begin{align*}
		R_{\text{PNNU}}^{\eta}(g) &= (1-\eta)(\theta_1 R_{11} + \theta_2 R_{22}) + \eta(R_{U1}-\theta_2 R_{21} + \theta_2 R_{22}) \\
		&= (1-\eta)\theta_1 R_{11} + \theta_2 R_{22} -\eta\theta_2R_{21}+\eta R_{U1},
	\end{align*}
	which is equivalent to $R_{lin}^{\boldsymbol{a}}(g)$ where $(a_1,a_2)=(\theta_1(1-\eta),\theta_2)$.
	
	The sets of estimators for PNPU and PNNU learning are subsets of the generalized linear estimator set $S_{lin}$. Thus, Eq. \eqref{eq:var_com} holds.
	
	\textbf{Derivation of PNPU Minimum Variance:}
	We explicitly derive the minimum variance for the PNPU estimator in the binary setting ($k=2$). The variance of the estimator $\hat{R}_{\text{PNPU}}^{\eta}$ when $n_U\rightarrow\infty$ is given by:
	\begin{align*}
		\operatorname{Var}(\hat{R}_{\text{PNPU}}^{\eta}) &= \operatorname{Var}\left( \theta_1 \hat{R}_{11} - \eta \theta_1 \hat{R}_{12} + (1-\eta) \theta_2 \hat{R}_{22} \right) \\
		&= \theta_1^2 \operatorname{Var}(\hat{R}_{11} - \eta \hat{R}_{12}) + (1-\eta)^2 \theta_2^2 \operatorname{Var}(\hat{R}_{22}).
	\end{align*}
	
	Assuming $C_1=C_2=C$ with diagonal elements $\rho_1$ and off-diagonal elements $\rho_2$:
	\begin{align*}
		\operatorname{Var}(\hat{R}_{11} - \eta \hat{R}_{12}) &= \operatorname{Var}(\hat{R}_{11}) - 2\eta \operatorname{Cov}(\hat{R}_{11}, \hat{R}_{12}) + \eta^2 \operatorname{Var}(\hat{R}_{12}) \\
		&= \frac{\rho_1}{n_1} - \frac{2\eta \rho_2}{n_1} + \frac{\eta^2 \rho_1}{n_1}.
	\end{align*}
	Substituting this back into the total variance expression and using $w_i = \theta_i^2/n_i$:
	\begin{align*}
		\operatorname{Var}(\hat{R}_{\text{PNPU}}^{\eta}) &= \theta_1^2 \left( \frac{\rho_1 - 2\eta \rho_2 + \eta^2 \rho_1}{n_1} \right) + (1-\eta)^2 \theta_2^2 \frac{\rho_1}{n_2} \\
		&= w_1 (\rho_1 - 2\eta \rho_2 + \eta^2 \rho_1) + w_2 (1 - 2\eta + \eta^2) \rho_1\\
		&= (w_1 + w_2)\rho_1 \eta^2 -2(w_1 \rho_2 + w_2 \rho_1) \eta + (w_1 + w_2)\rho_1 .
	\end{align*}
	
	Then, 
	\begin{align*}
		\operatorname{Var}(\hat{R}_{\text{PNPU}}^{\eta^*})&= (w_1 + w_2)\rho_1 - \frac{\left(-2(w_1 \rho_2 + w_2 \rho_1)\right)^2}{4(w_1 + w_2)\rho_1}\\
		&= (w_1 + w_2)\rho_1 - \frac{4(w_1 \rho_2 + w_2 \rho_1)^2}{4(w_1 + w_2)\rho_1}\\
		&= \frac{(w_1 + w_2)^2 \rho_1^2 - (w_1 \rho_2 + w_2 \rho_1)^2}{\rho_1(w_1 + w_2)}\\
		&=\frac{w_1^2 (\rho_1^2 - \rho_2^2) + 2 w_1 w_2 \rho_1 (\rho_1 - \rho_2)}{\rho_1 (w_1 + w_2)}.
	\end{align*}

	\textbf{Variance gap \eqref{eq:var_comval}:}
	Using the result from Theorem \ref{thm:varred} for $k=2$, the minimum variance of our proposed estimator is:
	\begin{align*}
		\operatorname{Var}(\hat{R}_{lin}^{\boldsymbol{a}^*}) &= \sum_{m=1}^2 w_m \rho_1 - \left[ (\rho_1 - \rho_2) \frac{w_1^2 + w_2^2}{W} + \rho_2 W \right] \\
		&= \rho_1 W - \rho_2 W - (\rho_1 - \rho_2) \frac{w_1^2 + w_2^2}{W} \\
		&= (\rho_1 - \rho_2) \left( W - \frac{w_1^2 + w_2^2}{W} \right) \\
		&= (\rho_1 - \rho_2) \frac{(w_1 + w_2)^2 - (w_1^2 + w_2^2)}{W} \\
		&= (\rho_1 - \rho_2) \frac{2 w_1 w_2}{w_1 + w_2}.
	\end{align*}
	We compute the gap $\operatorname{Var}(\hat{R}_{\text{PNPU}}^{\eta^*_{PU}}) - \operatorname{Var}(\hat{R}_{lin}^{\boldsymbol{a}^*})$:
	\begin{align*}
		\text{Gap} &= \frac{w_1^2 (\rho_1 - \rho_2)(\rho_1 + \rho_2) + 2 w_1 w_2 \rho_1 (\rho_1 - \rho_2)}{\rho_1 W} - \frac{2 w_1 w_2 (\rho_1 - \rho_2)}{W} \\
		&= \frac{(\rho_1 - \rho_2)}{W} \left[ \frac{w_1^2 (\rho_1 + \rho_2) + 2 w_1 w_2 \rho_1}{\rho_1} - 2 w_1 w_2 \right] \\
		&= \frac{(\rho_1 - \rho_2)}{W} \left[ \frac{w_1^2 (\rho_1 + \rho_2)}{\rho_1} + 2 w_1 w_2 - 2 w_1 w_2 \right] \\
		&= \frac{w_1^2 (\rho_1^2 - \rho_2^2)}{(w_1 + w_2)\rho_1}.
	\end{align*}
	
	By symmetry, the gap for PNNU is $\frac{w_2^2 (\rho_1^2 - \rho_2^2)}{(w_1 + w_2)\rho_1}$. The result follows.
\end{proof}

\begin{proof}[Proof of Theorem \ref{thm:rewrite_slin_sym}]
	From the definition of symmetric loss $\sum_{j=1}^k R_{ij} = \alpha$, the following holds:
	\begin{equation}
		R_{ik} = \alpha - \sum_{j=1}^{k-1} R_{ij}, \quad R_{Uk} = \alpha - \sum_{j=1}^{k-1} R_{Uj}
	\end{equation}
	
	Substituting the above expressions into $R_{lin}(g)$, we derive
	\begin{align*}
		R_{lin}^{\{a_{ij}\}, \{b_j\}}&= \sum_{i,j} a_{ij} R_{ij}(g) + \sum_{j} b_j R_{Uj} \\
		&= \sum_{i} \left[ \sum_{j=1}^{k-1} a_{ij} R_{ij} + a_{ik} \left( \alpha - \sum_{j=1}^{k-1} R_{ij} \right) \right] + \sum_{j=1}^{k-1} b_j R_{Uj} + b_k \left( \alpha - \sum_{j=1}^{k-1} R_{Uj} \right) \\
		&= \sum_{i} \sum_{j=1}^{k-1} \underbrace{(a_{ij} - a_{ik})}_{\tilde{a}_{ij}} R_{ij} + \sum_{j=1}^{k-1} \underbrace{(b_j - b_k)}_{\tilde{b}_{j}} R_{Uj} + \alpha\left(\sum_{i} a_{ik} + b_k\right)
	\end{align*}
	
	where we define $\tilde{a}_{ij}:=a_{ij} - a_{ik}$ and $\tilde{b}_j:=b_j - b_k$. Using $R_{Uj} = \sum_{i} \theta_i R_{ij}$, the second term can be rewritten as:
	\begin{equation}
		\sum_{j=1}^{k-1} (b_j - b_k) R_{Uj} = \sum_{j=1}^{k-1} (b_j - b_k) \sum_{i} \theta_i R_{ij} = \sum_{i} \sum_{j=1}^{k-1} (b_j - b_k) \theta_i R_{ij}
	\end{equation}

	Substituting these back into $R_{lin}$ and rearranging:
	\begin{equation}
		R_{lin}^{\{a_{ij}\}, \{b_j\}} = \sum_{i} \sum_{j=1}^{k-1} \left( a_{ij} - a_{ik} + (b_j - b_k)\theta_i \right) R_{ij} + \alpha \left( \sum_{i} a_{ik} + b_k \right)
	\end{equation}
	
	On the other hand, the true risk $R(g)$ can be written as follows (expanding $R_{ik}$):
	\begin{align*}
		R(g) &= \sum_{i} \theta_i R_{ii} = \sum_{i=1}^{k-1} \theta_i R_{ii} + \theta_k R_{kk} \\
		&= \sum_{i=1}^{k-1} \theta_i R_{ii} + \theta_k \left( \alpha - \sum_{j=1}^{k-1} R_{kj} \right)
	\end{align*}

	Under Assumption \ref{asmp:linindep_sym}, we derive the conditions for the identity $\forall g\in\mathcal{G}, R_{lin}(g) = R(g)$ to hold :
	
	\begin{equation}
		\begin{cases}
			\forall i, j \in [k-1], & a_{ij} - a_{ik} + (b_j - b_k)\theta_i - \delta_{ij}\theta_i = 0 \quad (\text{Coeff. of } R_{ij}) \\
			\forall j \in [k-1], & a_{kj} - a_{kk} + (b_j - b_k)\theta_k + \theta_k = 0 \quad (\text{Coeff. of } R_{kj}) \\
			& \sum_{i} a_{ik} + b_k -  \theta_k = 0 \quad (\text{Constant term})
		\end{cases}
	\end{equation}
	
	Suppose we are given arbitrary $\tilde{a}_{ij} \in \mathbb{R}$ and $\tilde{b}_j \in \mathbb{R}$ for all $i \in [k]$ and $j \in [k-1]$ that satisfy the above equations. There are values of $a_{ij}, b_j$ that satisfy the above equations and $\tilde{a}_{ij}=a_{ij} - a_{ik}$ and $\tilde{b}_j=b_j - b_k$. Thus, we can simplify $R_{lin}^{\{a_{ij}\}, \{b_j\}}(g)$ as
	\begin{equation}
		R_{lin}(g) = \sum_{i} \sum_{j}^{k-1} \tilde{a}_{ij} R_{ij} + \sum_{j}^{k-1} \tilde{b}_j R_{Uj} + \alpha \theta_k, \quad \text{where } \tilde{a}_{ij}, \tilde{b}_j \in \mathbb{R}
	\end{equation}
	
	The conditions are:
	\begin{enumerate}
		\item $\tilde{a}_{ij} + \tilde{b}_j \theta_i - \delta_{ij}\theta_i = 0 \quad \forall i, j \in [k-1]$
		\item $\tilde{a}_{kj} + \tilde{b}_j \theta_k + \theta_k = 0 \quad \forall j \in [k-1]$
	\end{enumerate}
	
	Solving these yields:
	\begin{align}
		\forall i \in [k-1], \quad & \tilde{a}_{ii} + \tilde{b}_i \theta_i - \theta_i = 0 \nonumber \\
		\Rightarrow \quad & \tilde{b}_i = 1 - \frac{\tilde{a}_{ii}}{\theta_i}
	\end{align}
	
	Also, when $i \neq j$, $\tilde{a}_{ij} + \tilde{b}_j \theta_i = 0$. From these, we get:
	\begin{align}
		\forall i,j \in [k-1] \text{ s.t. } i\neq j, \quad \tilde{a}_{ij} &=-\theta_i \tilde{b}_j \nonumber \\
		&= -\theta_i \left( 1 - \frac{\tilde{a}_{jj}}{\theta_j} \right)
	\end{align}
	
	Furthermore, for the $k$-th class:
	\begin{align}
		\tilde{a}_{kj} &= -\theta_k (1 + \tilde{b}_j) \nonumber \\
		&= -\theta_k \left( 2 - \frac{\tilde{a}_{jj}}{\theta_j} \right)
	\end{align}
	
	Substituting the above equations into $R_{lin}(g)$, we derive the target set.
\end{proof}

\begin{proof}[Proof of Theorem \ref{thm:var_sym}]
	
	We derive the variance of the estimator $\hat{R}_{lin}^{\boldsymbol{a}}$ under the symmetric loss assumption. Recall that the estimator is given by:
	\begin{equation}
		\hat{R}^{\boldsymbol{a}}_{lin} = \sum^k_{i=1}\theta_i \hat{R}_{ii} + \sum^k_{i=1}\sum^{k-1}_{j=1}\theta_i\left(\frac{a_j}{\theta_j}-1\right)\hat{R}_{ij} - \sum^{k-1}_{j=1} \left(\frac{a_j}{\theta_j}-1\right)\hat{R}_{Uj}.
	\end{equation}
	In the asymptotic regime where $n_U \to \infty$, the variance of the unlabeled component vanishes. Thus, we focus on the labeled components. Since the labeled datasets $\mathcal{X}_m$ for $m \in \{1, \dots, k\}$ are independent, the total variance is the sum of the variances contributed by each class. We define $T_m$ as the terms in $\hat{R}^{\boldsymbol{a}}_{lin}$ dependent on $\mathcal{X}_m$:
	\begin{equation}
		T_m = \theta_m \hat{R}_{mm} + \sum_{j=1}^{k-1} \theta_m \left(\frac{a_j}{\theta_j}-1\right) \hat{R}_{mj}.
	\end{equation}
	Let $\check{\mathbf{R}}_m = [\hat{R}_{m1}, \dots, \hat{R}_{m, k-1}]^\top$ be the vector of component risks for class $m$ restricted to the first $k-1$ labels. The covariance matrix of $\check{\mathbf{R}}_m$ is $\frac{1}{n_m}\check{C}_m$. We express $T_m$ as a linear combination $\mathbf{c}_m^\top \check{\mathbf{R}}_m + \text{const}$ and determine the coefficient vector $\mathbf{c}_m \in \mathbb{R}^{k-1}$.
	
	\textbf{Case 1: $m < k$.}
	For $m \in \{1, \dots, k-1\}$, the term $\hat{R}_{mm}$ is explicitly contained in the sum $\sum_{j=1}^{k-1}$.
	The coefficient for $\hat{R}_{mj}$ where $j \neq m$ is $\theta_m (a_j/\theta_j - 1)$.
	The coefficient for $\hat{R}_{mm}$ combines the first term and the summation term: $\theta_m + \theta_m (a_m/\theta_m - 1) = a_m$.
	Thus, the $j$-th element of $\mathbf{c}_m$ is:
	\begin{equation}
		(\mathbf{c}_m)_j = \begin{cases} \theta_m \left(\frac{a_j}{\theta_j} - 1\right) & j \neq m \\ \theta_m \left(\frac{a_m}{\theta_m}\right) & j = m \end{cases}.
	\end{equation}
	Using the diagonal matrix $\check{Q} = \operatorname{diag}(1/\theta_1, \dots, 1/\theta_{k-1})$ and vector $\check{\mathbf{d}}_m = \mathbf{1} - \mathbf{e}_m$, we can write this compactly as:
	\begin{equation}
		\mathbf{c}_m = \theta_m (\check{Q}\boldsymbol{a} - \check{\mathbf{d}}_m).
	\end{equation}
	
	\textbf{Case 2: $m = k$.}
	For class $k$, $\hat{R}_{kk}$ is not included in $\check{\mathbf{R}}_k$. We invoke the symmetric loss property $\sum_{j=1}^k \hat{R}_{kj} = \alpha$ (constant) to substitute $\hat{R}_{kk} = \alpha - \sum_{j=1}^{k-1} \hat{R}_{kj}$.
	Substituting this into $T_k$:
	\begin{align}
		T_k &= \theta_k \left(\alpha - \sum_{j=1}^{k-1} \hat{R}_{kj}\right) + \sum_{j=1}^{k-1} \theta_k \left(\frac{a_j}{\theta_j}-1\right) \hat{R}_{kj} \\
		&= \text{const} + \sum_{j=1}^{k-1} \left[ -\theta_k + \theta_k \left(\frac{a_j}{\theta_j}-1\right) \right] \hat{R}_{kj} \\
		&= \text{const} + \sum_{j=1}^{k-1} \theta_k \left(\frac{a_j}{\theta_j}-2\right) \hat{R}_{kj}.
	\end{align}
	Here, the coefficient vector corresponds to $\check{\mathbf{d}}_k = 2\mathbf{1}$. Thus:
	\begin{equation}
		\mathbf{c}_k = \theta_k (\check{Q}\boldsymbol{a} - 2\mathbf{1}) = \theta_k (\check{Q}\boldsymbol{a} - \check{\mathbf{d}}_k).
	\end{equation}
	
	\textbf{Total Variance.}
	The variance of the estimator is the sum of variances of $T_m$:
	\begin{align}
		\operatorname{Var}\left[\hat{R}_{lin}^{\boldsymbol{a}}\right] &= \sum_{m=1}^k \operatorname{Var}[T_m] = \sum_{m=1}^k \mathbf{c}_m^\top \left( \frac{1}{n_m} \check{C}_m \right) \mathbf{c}_m \\
		&= \sum_{m=1}^k \frac{\theta_m^2}{n_m} (\check{Q}\boldsymbol{a} - \check{\mathbf{d}}_m)^\top \check{C}_m (\check{Q}\boldsymbol{a} - \check{\mathbf{d}}_m).
	\end{align}
\end{proof}

\begin{proof}[Proof of Theorem \ref{thm:min_var_sym}]
	First, we simplify the reduction term $\check{\mathbf{b}}^T \check{A}^{-1} \check{\mathbf{b}}$. Substituting $\check{A}^{-1} = \check{Q}^{-1} \check{S}^{-1} \check{Q}^{-1}$ and $\check{\mathbf{b}} = \check{Q}(\check{S}\mathbf{1} - \check{\mathbf{u}})$:
	\begin{align*}
		\check{\mathbf{b}}^T \check{A}^{-1} \check{\mathbf{b}} &= (\check{S}\mathbf{1} - \check{\mathbf{u}})^T \check{S}^{-1} (\check{S}\mathbf{1} - \check{\mathbf{u}}) \\
		&= (\mathbf{1}^T \check{S} - \check{\mathbf{u}}^T) (\mathbf{1} - \check{S}^{-1}\check{\mathbf{u}}) \\
		&= \mathbf{1}^T \check{S} \mathbf{1} - 2\mathbf{1}^T \check{\mathbf{u}} + \check{\mathbf{u}}^T \check{S}^{-1} \check{\mathbf{u}}.
	\end{align*}
	
	Next, we expand the constant term $\check{c}$. Recalling $\check{\mathbf{d}}_m = \mathbf{1} - \mathbf{e}_m$ for $m < k$ and $\check{\mathbf{d}}_k = 2\mathbf{1}$:
	\begin{align*}
		\check{c} &= \sum_{m=1}^{k-1} \frac{\theta_m^2}{n_m} (\mathbf{1} - \mathbf{e}_m)^T \check{C}_m (\mathbf{1} - \mathbf{e}_m) + \frac{\theta_k^2}{n_k} (2\mathbf{1})^T \check{C}_k (2\mathbf{1}) \\
		&= \sum_{m=1}^{k-1} \frac{\theta_m^2}{n_m} \left( \mathbf{1}^T \check{C}_m \mathbf{1} - 2\mathbf{1}^T \check{C}_m \mathbf{e}_m + (\check{C}_m)_{mm} \right) + 4\frac{\theta_k^2}{n_k} \mathbf{1}^T \check{C}_k \mathbf{1}.
	\end{align*}
	We regroup the terms using $\check{S} = \sum_{m=1}^{k-1} \frac{\theta_m^2}{n_m} \check{C}_m + \frac{\theta_k^2}{n_k} \check{C}_k$:
	\begin{itemize}
		\item Quadratic terms in $\mathbf{1}$: 
		$$ \mathbf{1}^T \left(\sum_{m=1}^{k-1} \frac{\theta_m^2}{n_m} \check{C}_m\right) \mathbf{1} + 4\frac{\theta_k^2}{n_k} \mathbf{1}^T \check{C}_k \mathbf{1} = \mathbf{1}^T (\check{S} - \frac{\theta_k^2}{n_k} \check{C}_k) \mathbf{1} + 4\frac{\theta_k^2}{n_k} \mathbf{1}^T \check{C}_k \mathbf{1} = \mathbf{1}^T \check{S} \mathbf{1} + 3\frac{\theta_k^2}{n_k} \mathbf{1}^T \check{C}_k \mathbf{1}.$$
		\item Linear terms in $\mathbf{1}$: 
		$$ -2 \mathbf{1}^T \sum_{m=1}^{k-1} \frac{\theta_m^2}{n_m} \check{C}_m \mathbf{e}_m = -2 \mathbf{1}^T (\check{\mathbf{u}} + \frac{\theta_k^2}{n_k} \check{C}_k \mathbf{1}) = -2 \mathbf{1}^T \check{\mathbf{u}} - 2\frac{\theta_k^2}{n_k} \mathbf{1}^T \check{C}_k \mathbf{1}.$$
		\item Remaining diagonal terms: $\sum_{m=1}^{k-1} \frac{\theta_m^2}{n_m} (\check{C}_m)_{mm}$.
	\end{itemize}
	Thus:
	\begin{align*}
		\check{c} &= \mathbf{1}^T \check{S} \mathbf{1} - 2 \mathbf{1}^T \check{\mathbf{u}} + \left( \sum_{m=1}^{k-1} \frac{\theta_m^2}{n_m} (\check{C}_m)_{mm} + \frac{\theta_k^2}{n_k} \mathbf{1}^T \check{C}_k \mathbf{1} \right)\\
		&=\mathbf{1}^T \check{S} \mathbf{1} - 2 \mathbf{1}^T \check{\mathbf{u}} + \sum_{m=1}^{k} \frac{\theta_m^2}{n_m} (C_m)_{mm}
	\end{align*}
	Subtracting the reduction term from $\check{c}$ cancels the $\mathbf{1}^T \check{S} \mathbf{1}$ and linear terms, yielding the result.
\end{proof}

\begin{proof}[Proof of Theorem \ref{thm:equiv}]
	First, we derive the explicit forms of the risks under the symmetric loss assumption. 
	Recall the symmetric condition implies $R_{12} = \alpha - R_{11}$ and $R_{21} = \alpha - R_{22}$.
	Substituting these into the definitions of PU and NU risks, we obtain:
	\begin{align}
		R_{\text{PU}}(g) &= \theta_1 R_{11}(g) - \theta_1 (\alpha - R_{11}(g)) + R_{U2}(g) \nonumber \\
		&= 2\theta_1 R_{11}(g) + R_{U2}(g) - \theta_1 \alpha, \label{eq:R_PU_sym} \\
		R_{\text{NU}}(g) &= \theta_2 R_{22}(g) - \theta_2 (\alpha - R_{22}(g)) + R_{U1}(g) \nonumber \\
		&= 2\theta_2 R_{22}(g) + R_{U1}(g) - \theta_2 \alpha. \label{eq:R_NU_sym}
	\end{align}
	
	Next, we consider the PNPU estimator $R_{\text{PNPU}}^{\eta} = (1-\eta)R_{\text{PN}} + \eta R_{\text{PU}}$. Substituting \eqref{eq:R_PU_sym}:
	\begin{align}
		R_{\text{PNPU}}^{\eta}(g) &= (1-\eta)(\theta_1 R_{11} + \theta_2 R_{22}) + \eta (2\theta_1 R_{11} + R_{U2} - \theta_1 \alpha) \nonumber \\
		&= \theta_1(1+\eta) R_{11} + \theta_2(1-\eta) R_{22} + \eta R_{U2} - \eta \theta_1 \alpha. \label{eq:PNPU_form}
	\end{align}
	Similarly, for the PNNU estimator $R_{\text{PNNU}}^{\eta} = (1-\eta)R_{\text{PN}} + \eta R_{\text{NU}}$, using \eqref{eq:R_NU_sym}:
	\begin{align}
		R_{\text{PNNU}}^{\eta}(g) &= (1-\eta)(\theta_1 R_{11} + \theta_2 R_{22}) + \eta (2\theta_2 R_{22} + R_{U1} - \theta_2 \alpha) \nonumber \\
		&= \theta_1(1-\eta) R_{11} + \theta_2(1+\eta) R_{22} + \eta R_{U1} - \eta \theta_2 \alpha. \label{eq:PNNU_form}
	\end{align}
	
	Now, we examine the general linear estimator set $S_{lin}$ for $k=2$ with symmetric loss. From Theorem \ref{thm:rewrite_slin_sym}, the estimator is parametrized by a scalar $a_1$:
	\begin{equation}
		R_{lin}^{a_1}(g) = a_1 R_{11} + \theta_2 R_{22} + \theta_2 \left(\frac{a_1}{\theta_1}-1\right) R_{21} - \left(\frac{a_1}{\theta_1}-1\right) R_{U1}.
	\end{equation}
	Using the identity $R_{21} = \alpha - R_{22}$ and $R_{U1} + R_{U2} = \alpha$ to align terms with \eqref{eq:PNPU_form}, we rewrite $R_{lin}^{a_1}$:
	\begin{align*}
		R_{lin}^{a_1}(g) &= a_1 R_{11} + \theta_2 R_{22} + \theta_2 \left(\frac{a_1}{\theta_1}-1\right) (\alpha - R_{22}) - \left(\frac{a_1}{\theta_1}-1\right) (\alpha - R_{U2}) \\
		&= a_1 R_{11} + \left[ \theta_2 - \theta_2\left(\frac{a_1}{\theta_1}-1\right) \right] R_{22} + \left(\frac{a_1}{\theta_1}-1\right) R_{U2} + \text{const.} \\
		&= a_1 R_{11} + \theta_2 \left( 2 - \frac{a_1}{\theta_1} \right) R_{22} + \left(\frac{a_1}{\theta_1}-1\right) R_{U2} + \text{const.}
	\end{align*}
	Comparing this with the PNPU form in \eqref{eq:PNPU_form}, we see that they match if we set $a_1 = \theta_1(1+\eta)$. Specifically:
	\begin{itemize}
		\item Coefficient of $R_{11}$: $a_1 \iff \theta_1(1+\eta)$
		\item Coefficient of $R_{22}$: $\theta_2(2 - (1+\eta)) = \theta_2(1-\eta)$
		\item Coefficient of $R_{U2}$: $(1+\eta) - 1 = \eta$
	\end{itemize}
	Since this mapping is bijective for $\eta, a_1 \in \mathbb{R}$, we have $S_{lin} = S_{\text{PNPU}}$.
	
	Similarly, comparing $R_{lin}^{a_1}$ with PNNU form in \eqref{eq:PNNU_form} by setting $a_1 = \theta_1(1-\eta)$ shows that $S_{lin} = S_{\text{PNNU}}$. Thus, all three sets are equivalent.
\end{proof}

\begin{proof}[Proof of Theorem \ref{thm:gen_bound}]
	We consider a hypothesis class $\mathcal{G}$ and its minimal $\nu$-cover $C_\nu$ with respect to the $L_\infty$ norm. By definition, $|C_\nu| = \mathcal{N}(\mathcal{G}, \nu, \|\cdot\|_\infty)$, and for any $g \in \mathcal{G}$, there exists a $g' \in C_\nu$ such that $\|g - g'\|_\infty \leq \nu$.
	
	Let $\hat{g}$ be the empirical risk minimizer and $g^*$ be the true risk minimizer. As $n_U \to \infty$, the variance of the unlabeled risk components vanishes, so we focus on the empirical estimates derived from the labeled datasets.
	
	\textbf{1. Decomposition of Excess Risk and Covering Bound}
	First, we relate the excess risk to the maximum estimation error over $\mathcal{G}$:
	\begin{align}
		R(\hat{g}) - R(g^*) &= R(\hat{g}) - \hat{R}_{lin}^{\boldsymbol{a}}(\hat{g}) + \hat{R}_{lin}^{\boldsymbol{a}}(\hat{g}) - \hat{R}_{lin}^{\boldsymbol{a}}(g^*) + \hat{R}_{lin}^{\boldsymbol{a}}(g^*) - R(g^*) \nonumber \\
		&\leq R(\hat{g}) - \hat{R}_{lin}^{\boldsymbol{a}}(\hat{g}) + \hat{R}_{lin}^{\boldsymbol{a}}(g^*) - R(g^*) \quad (\text{since } \hat{R}_{lin}^{\boldsymbol{a}}(\hat{g}) \leq \hat{R}_{lin}^{\boldsymbol{a}}(g^*)) \nonumber \\
		&\leq 2 \sup_{g \in \mathcal{G}} |\hat{R}_{lin}^{\boldsymbol{a}}(g) - R(g)|.
	\end{align}
	
	Using the $\nu$-cover $C_\nu$, we can bound the deviation for any $g \in \mathcal{G}$ by mapping it to its closest element $g' \in C_\nu$. Because the loss function is $L$-Lipschitz continuous, we have $|l(g(x), j) - l(g'(x), j)| \leq L\|g(x) - g'(x)\|_\infty \leq L\nu$. 
	
	For the empirical risk estimator, the difference is bounded by:
	\begin{align*}
		|\hat{R}_{lin}^{\boldsymbol{a}}(g) - \hat{R}_{lin}^{\boldsymbol{a}}(g')| &\leq \sum_{m=1}^k \frac{\theta_m}{n_m} \sum_{q=1}^{n_m} \left( |l(g(x_q^m), m) - l(g'(x_q^m), m)| + \sum_{j=1}^k \left|\frac{a_j}{\theta_j} - 1\right| |l(g(x_q^m), j) - l(g'(x_q^m), j)| \right) \\
		&\leq \sum_{m=1}^k \theta_m \left( L\nu + k c_{\boldsymbol{a}} L\nu \right) = L\nu(1 + k c_{\boldsymbol{a}}) = L_{\boldsymbol{a}}\nu.
	\end{align*}
	
	Similarly, for the true risk, $|R(g) - R(g')| \leq \mathbb{E}[|l(g(x), y) - l(g'(x), y)|] \leq L\nu \leq L_{\boldsymbol{a}}\nu$.
	Thus, using the triangle inequality for any fixed $g \in \mathcal{G}$ and its closest element $g' \in C_\nu$ (where $\|g - g'\|_\infty \leq \nu$), we have:
	\begin{align}
		|\hat{R}_{lin}^{\boldsymbol{a}}(g) - R(g)| &\leq |\hat{R}_{lin}^{\boldsymbol{a}}(g) - \hat{R}_{lin}^{\boldsymbol{a}}(g')| + |\hat{R}_{lin}^{\boldsymbol{a}}(g') - R(g')| + |R(g') - R(g)| \nonumber \\
		&\leq 2L_{\boldsymbol{a}}\nu + |\hat{R}_{lin}^{\boldsymbol{a}}(g') - R(g')| \nonumber \\
		&\leq 2L_{\boldsymbol{a}}\nu + \max_{g'' \in C_\nu} |\hat{R}_{lin}^{\boldsymbol{a}}(g'') - R(g'')|.
	\end{align}
	
	Since the right-hand side is independent of the choice of $g$, we can now take the supremum over all $g \in \mathcal{G}$ on the left-hand side:
	\begin{equation}
		\sup_{g \in \mathcal{G}} |\hat{R}_{lin}^{\boldsymbol{a}}(g) - R(g)| \leq \max_{g' \in C_\nu} |\hat{R}_{lin}^{\boldsymbol{a}}(g') - R(g')| + 2L_{\boldsymbol{a}}\nu.
		\label{eq:cover_bound}
	\end{equation}
	
	\textbf{2. Bernstein's Inequality on the Finite Cover}
	We now bound the estimation error for the finite set of hypotheses in $C_\nu$. For any fixed $g' \in C_\nu$, the estimator $\hat{R}_{lin}^{\boldsymbol{a}}(g')$ is a sum of independent random variables:
	\begin{equation}
		\hat{R}^a_{lin}(g') = \sum_{m=1}^k \sum_{q=1}^{n_m} \underbrace{\frac{\theta_m}{n_m} \left( l(g'(x_q^m), m) + \sum_{j=1}^k \left( \frac{a_j}{\theta_j} - 1 \right) l(g'(x_q^m), j) \right)}_{Z_q^m(g')} + \text{const}.
	\end{equation}
	Since $0 \le l(\cdot, \cdot) \le c_l$, the variables $Z_q^m(g')$ are bounded by $\max_{m\in[k]}\left(\frac{\theta_m}{n_m}\right) k c_l(1+c_{\boldsymbol{a}}) = \max_{m\in[k]}\left(\frac{\theta_m}{n_m}\right) B_{\boldsymbol{a}}$.
	
	First, for any fixed $g' \in C_\nu$, applying Bernstein's inequality yields:
	\begin{equation}
		P\left( |\hat{R}_{lin}^{\boldsymbol{a}}(g') - R(g')| > \epsilon \right) \leq 2 \exp\left( \frac{-\epsilon^2}{2 \operatorname{Var}[\hat{R}_{lin}^{\boldsymbol{a}}(g')] + \frac{2}{3} \max_{m\in[k]}\left(\frac{\theta_m}{n_m}\right) B_{\boldsymbol{a}}\epsilon} \right).
	\end{equation}
	Since $\operatorname{Var}[\hat{R}_{lin}^{\boldsymbol{a}}(g')] \leq \sigma_{\max}^2(\boldsymbol{a}, \nu)$ for all $g' \in C_\nu$, we can upper bound this probability by:
	\begin{equation}
		P\left( |\hat{R}_{lin}^{\boldsymbol{a}}(g') - R(g')| > \epsilon \right) \leq 2 \exp\left( \frac{-\epsilon^2}{2 \sigma_{\max}^2(\boldsymbol{a}, \nu) + \frac{2}{3} \max_{m\in[k]}\left(\frac{\theta_m}{n_m}\right) B_{\boldsymbol{a}}\epsilon} \right).
	\end{equation}
	
	Next, applying the union bound over all $|C_\nu| = \mathcal{N}(\mathcal{G}, \nu, \|\cdot\|_\infty)$ elements, we obtain:
	\begin{align}
		P\left( \max_{g' \in C_\nu} |\hat{R}_{lin}^{\boldsymbol{a}}(g') - R(g')| > \epsilon \right) 
		&= P\left( \bigcup_{g' \in C_\nu} \left\{ |\hat{R}_{lin}^{\boldsymbol{a}}(g') - R(g')| > \epsilon \right\} \right) \nonumber \\
		&\leq \sum_{g' \in C_\nu} P\left( |\hat{R}_{lin}^{\boldsymbol{a}}(g') - R(g')| > \epsilon \right) \nonumber \\
		&\leq 2\mathcal{N}(\mathcal{G}, \nu, \|\cdot\|_\infty) \exp\left( \frac{-\epsilon^2}{2 \sigma_{\max}^2(\boldsymbol{a}, \nu) + \frac{2}{3} \max_{m\in[k]}\left(\frac{\theta_m}{n_m}\right) B_{\boldsymbol{a}}\epsilon} \right).
	\end{align}
	
	
	Equating the right-hand side of the union bound to $\delta$, we find the bound $\epsilon$ that holds with probability at least $1-\delta$:$$\frac{\epsilon^2}{2 \sigma_{\max}^2(\boldsymbol{a}, \nu) + \frac{2}{3} \max_{m\in[k]}\left(\frac{\theta_m}{n_m}\right) B_{\boldsymbol{a}} \epsilon} = \ln\left(\frac{2\mathcal{N}(\mathcal{G}, \nu, \|\cdot\|_\infty)}{\delta}\right).$$
	
	Rearranging terms yields a quadratic equation for $\epsilon$:$$\epsilon^2 - \underbrace{\left( \frac{2}{3} \max_{m\in[k]}\left(\frac{\theta_m}{n_m}\right) B_{\boldsymbol{a}} \ln\left(\frac{2\mathcal{N}(\mathcal{G}, \nu, \|\cdot\|_\infty)}{\delta}\right) \right)}_{b} \epsilon - \underbrace{\left( 2 \sigma_{\max}^2(\boldsymbol{a}, \nu) \ln\left(\frac{2\mathcal{N}(\mathcal{G}, \nu, \|\cdot\|_\infty)}{\delta}\right) \right)}_{c} = 0.$$
	
	Using the quadratic formula, the positive root is $\epsilon = \frac{b + \sqrt{b^2 + 4c}}{2}$. By the subadditivity of the square root function ($\sqrt{x+y} \leq \sqrt{x} + \sqrt{y}$ for $x, y \geq 0$), we can simplify this as: $$\epsilon \leq \frac{b + b + 2\sqrt{c}}{2} = b + \sqrt{c}.$$
	
	Substituting the definitions of $b$ and $c$ back into the inequality, we obtain that with probability at least $1-\delta$:$$\max_{g' \in C_\nu} |\hat{R}_{lin}^{\boldsymbol{a}}(g') - R(g')| \leq \sqrt{2 \sigma_{\max}^2(\boldsymbol{a}, \nu) \ln\left(\frac{2\mathcal{N}(\mathcal{G}, \nu, \|\cdot\|_\infty)}{\delta}\right)} + \frac{2}{3} B_{\boldsymbol{a}} \max_{m\in[k]}\left(\frac{\theta_m}{n_m}\right) \ln\left(\frac{2\mathcal{N}(\mathcal{G}, \nu, \|\cdot\|_\infty)}{\delta}\right).$$
	
	\textbf{3. Final PAC Bound Derivation}
	Combining this result with the decomposition from Step 1 and Eq. \eqref{eq:cover_bound}, we bound the total excess risk:
	\begin{align*}
		R(\hat{g}) - R(g^*) &\leq 4L_{\boldsymbol{a}}\nu + 2 \max_{g' \in C_\nu} |\hat{R}_{lin}^{\boldsymbol{a}}(g') - R(g')| \\
		&\leq 4L_{\boldsymbol{a}}\nu + \sqrt{8 \sigma_{\max}^2(\boldsymbol{a}, \nu) \ln\left(\frac{2\mathcal{N}(\mathcal{G}, \nu, \|\cdot\|_\infty)}{\delta}\right)} + \frac{4}{3} B_{\boldsymbol{a}} \max_{m\in[k]}\left(\frac{\theta_m}{n_m}\right) \ln\left(\frac{2\mathcal{N}(\mathcal{G}, \nu, \|\cdot\|_\infty)}{\delta}\right) \\
	\end{align*}
	This completes the proof.
\end{proof}

\section{Variance comparison results for multiclass setting}

We present the variance comparison results for the multiclass setting. The procedure is analogous to the binary case described in Section~\ref{subsec:varcomp}. Since PNU learning is limited to binary classification, we compare only the supervised risk estimator $\hat{R}(g)$ and the proposed linear risk estimator $\hat{R}_{lin}^{\boldsymbol{a}^*}(g)$.

\paragraph{Datasets.} We use two multiclass datasets from the UCI Machine Learning Repository: Covertype and Dry Beans. For each dataset, we use 20\% of the data for training and 80\% as a held-out pool for evaluation.

\paragraph{Model and training.}
We train a two-layer neural network with hidden dimension 256 by minimizing the cross-entropy loss with $n_i = 10$ labeled samples per class via the Adam optimizer (learning rate $10^{-3}$, weight decay $10^{-4}$), with early stopping on validation accuracy.

\paragraph{Evaluation loss.}
We use the cross-entropy (CE) loss $l(g(x), j) = -\log \mathrm{softmax}(g(x))_j$ for the variance comparison.

\paragraph{Evaluation protocol.} With the trained classifier $g$ fixed, we evaluate the variance of each risk estimator by repeatedly sampling from the held-out data. The labeled sample sizes for risk estimation are fixed at $n_i = 5$ for all $i \in [k]$, and we vary $n_U \in \{100, 200, 500, 1000, 2000\}$. We consider two class prior settings: (1)~\emph{uniform}, where $\boldsymbol{\theta} = (1/7, \dots, 1/7)$, and (2)~\emph{skewed}, where $\boldsymbol{\theta}=[57/140, 34/140, 21/140, 13/140, 7/140, 5/140, 3/140] \approx (0.41,\; 0.24,\; 0.15,\; 0.09,\; 0.05,\; 0.04,\; 0.02)$. The optimal parameter $\boldsymbol{a}^*$ is computed assuming $n_U \to \infty$ using the covariance matrices estimated from the all evaluation data. For each condition, we conduct $5{,}000$ independent trials to estimate the variance.



\begin{figure}[htbp]
	\centering
	\begin{subfigure}[b]{0.48\columnwidth}
		\centering
		\includegraphics[width=\textwidth]{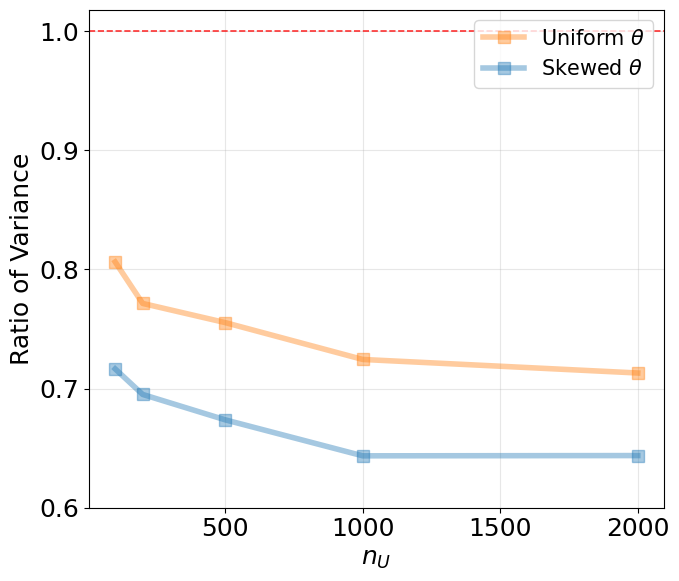}
		\vspace{-1.8em}
		\caption{Covertype (CE loss)}
		\label{fig:cov_ce}
	\end{subfigure}
	\hfill
	\begin{subfigure}[b]{0.48\columnwidth}
		\centering
		\includegraphics[width=\textwidth]{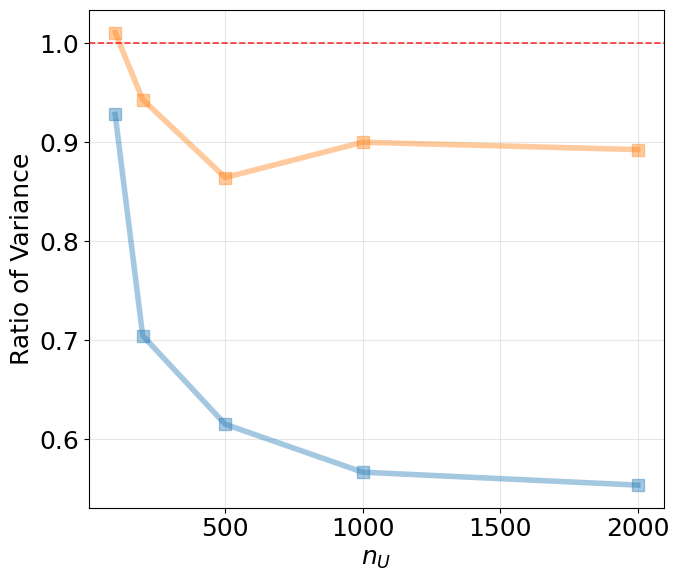}
		\vspace{-1.8em}
		\caption{Dry Beans (CE loss)}
		\label{fig:drybean_ce}
	\end{subfigure}
	
	\caption{Variance ratio ($\hat{R}_{lin}^{\boldsymbol{a}^*}$ / $\hat R$) as a function of the unlabeled sample size $n_U$ for uniform and skewed class priors ($k=7$). The red dashed line indicates the supervised baseline ($\mathrm{ratio} = 1$).} 
	\label{fig:var_compare_mult}
\end{figure}

As the number of unlabeled samples $n_U$ increases, the variance ratio decreases. The proposed method achieves greater variance reduction under the skewed class prior, consistent with Theorem~\ref{thm:varred}.

\section{SSL Experimental details}

We provide full details omitted from the main text for the SSL experiments in Section \ref{subsec:sslexp}.

\paragraph{Validation set.}
For all settings, the validation set is independently sampled from the original dataset.

\paragraph{Labeled data configurations for multiclass tasks.}
The per-class labeled sample counts for each regime are as follows.

\emph{7-class tasks} (Covertype, Shuttle, Dry Beans):
\begin{itemize}[nosep]
	\item Balanced: 10, 20, or 50 per class (totals: 70, 140, 350).
	\item Mild imbalance: $[33, 26, 23, 19, 16, 13, 10]$ (total 140) and $[81, 65, 57, 49, 41, 33, 24]$ (total 350).
	\item Severe imbalance: $[57, 34, 21, 13, 7, 5, 3]$ (total 140) and $[143, 86, 52, 32, 19, 11, 7]$ (total 350).
\end{itemize}

\emph{10-class tasks} (MNIST):
\begin{itemize}[nosep]
	\item Balanced: 20, 50, or 150 per class (totals: 200, 500, 1500).
	\item Mild imbalance: $[102, 82, 71, 61, 51, 41, 31, 26, 20, 15]$ (total 500) and $[306, 245, 214, 184, 153, 122, 92, 77, 61, 46]$ (total 1500).
	\item Severe imbalance: $[200, 120, 72, 44, 26, 16, 10, 6, 4, 2]$ (total 500) and $[600, 360, 216, 132, 78, 48, 30, 18, 12, 6]$ (total 1500).
\end{itemize}
For all multiclass tasks, $n_U = 5{,}000$.

\paragraph{Model architectures.}
\emph{MLP (tabular):}
We use an MLP with two hidden layers for all tabular datasets, using ReLU activations and a dropout rate of $0.2$. For the Adult, Banknote, Breast Cancer, and Credit Default datasets, we use hidden layer dimensions of $[256,256]$. For Covertype, we use $[512, 512]$, and for Shuttle and Dry Beans, we use $[512, 256]$.

\emph{CNN (image, LeNet-like):}
Here, let $\text{Conv}(c_{in}, c_{out}, k, \text{pad})$ denote a 2D convolutional layer with $c_{in}$ input channels, $c_{out}$ output channels, kernel size $k\times k$, and the specified padding. Let $\text{MaxPool}(k)$ denote a max pooling layer with kernel size $k\times k$. Let $\text{FC}(d_{in}, d_{out})$ denote a fully connected layer with input dimension $d_{in}$ and output dimension $d_{out}$. 
For MNIST, the architecture is $\text{Conv}(1, 32, 3, \text{pad}{=}1)$ $\to$ ReLU $\to$ $\text{MaxPool}(2)$ $\to$ $\text{Conv}(32, 64, 3, \text{pad}{=}1)$ $\to$ ReLU $\to$ $\text{MaxPool}(2)$ $\to$ $\text{FC}(64{\times}7{\times}7, 128)$ $\to$ ReLU $\to$ $\text{FC}(128, K)$.
For CIFAR-10, the architecture is $\text{Conv}(3, 64, 3, \text{pad}{=}1)$ $\to$ ReLU $\to$ $\text{MaxPool}(2)$ $\to$ $\text{Conv}(64, 128, 3, \text{pad}{=}1)$ $\to$ ReLU $\to$ $\text{MaxPool}(2)$ $\to$ $\text{FC}(128{\times}8{\times}8, 256)$ $\to$ ReLU $\to$ $\text{FC}(256, K)$.

\paragraph{Training protocol.}
All models are trained with the Adam optimizer ($\text{learning rate} =10^{-3}$, weight decay $10^{-4}$) for up to 200 epochs.
Batch sizes are 64 for labeled data and 256 for unlabeled data.
Early stopping is based on validation accuracy with patience 20 epochs and minimum improvement threshold $10^{-4}$.

\paragraph{Method-specific settings.}
\emph{PNU} (binary only): The mixing parameter $\eta$ is computed as $\eta = (\psi_2 - \psi_1)/(\psi_1 + \psi_2)$ where $\psi_m = \theta_m^2 / n_m$, following the equal variance assumption \citep{sakai17semi}. Non-negative risk correction ($\Delta_j^{nn}$) is applied.

\emph{Ours (Iter):} Training begins with a $20$-epoch warm-up period utilizing standard supervised risk prior to joint optimization. The covariance matrices are estimated on the validation split with diagonal shrinkage $\hat\Sigma_m \leftarrow (1-\alpha)\hat\Sigma_m + \alpha\,\mathrm{diag}(\hat\Sigma_m)$ ($\alpha=0.5$), and Tikhonov regularization $\lambda=10^{-4}$ is added when solving the linear system for the optimal $\boldsymbol{a}$. Non-negative risk correction ($\Delta_j^{nn}$) is applied.

\emph{Ours (EC):} The closed-form parameter $a_j = \theta_j (1 - w_j / W)$ with $w_j = \theta_j^2/n_j$ and $W = \sum_m w_m$ is used, requiring no validation data for parameter estimation.
The parameter $\boldsymbol{a}$ is fixed throughout training (no warmup). Non-negative risk correction ($\Delta_j^{nn}$) is applied.

\emph{Pseudo Label (PL):} The loss is the sum of a standard supervised cross-entropy term and an unsupervised CE term computed from pseudo labels. Initially, a supervised classifier is trained, followed by generating predictions on unlabeled data to create pseudo labels for instances exceeding a specific confidence threshold. The classifier is then retrained using both the original and pseudo labels. The confidence threshold is selected from the set $\{0.8, 0.9, 0.95\}$ for tabular multiclass tasks and $\{0.9, 0.95, 0.99\}$ for image tasks. 

\emph{VAT:} Supervised cross-entropy plus KL-divergence-based local distributional smoothness regularization with the consistency regularization weight is $\lambda_U = 1.0$, the finite-difference step size for approximating the adversarial direction is $\xi = 10^{-6}$, and $1$ power-iteration step is used. Batch-normalization running statistics are frozen during the adversarial perturbation computation.
The perturbation radius $\epsilon$ is selected from $\{0.2, 0.5, 1.0, 2.0\}$ for tabular tasks and $\{0.5, 1.0, 2.0, 4.0\}$ for image tasks.

\section{Additional experimental results}
This section provides the experimental details and ablation studies that were omitted from the main text because of space constraints. Unless otherwise stated, all table entries report mean$\pm$std over 30 random seeds. The section is organized as follows. Section~\ref{subsec:app-fixmatch} reports results for FixMatch, a recent semi-supervised learning baseline. Section~\ref{subsec:app-cov-size} studies how the number of examples used for covariance estimation affects our method. Section~\ref{subsec:app-homoscedasticity} reports diagnostics for the equal-covariance approximation used by the proposed estimators. Finally, Section~\ref{subsec:app-nn-ablation} studies the effect of the non-negative risk correction.

\subsection{Results for a recent SSL baseline}
\label{subsec:app-fixmatch}
We report additional results for FixMatch \citep{kihy20fixmatch}, a representative recent SSL method based on pseudo-labeling and consistency regularization. For each unlabeled example, FixMatch first predicts a pseudo-label from a weakly augmented view. It then applies an unsupervised cross-entropy loss to a strongly augmented view of the same example, but only when the maximum predicted class probability exceeds a confidence threshold.
Because this threshold can substantially affect performance, we sweep the threshold over $\{0.8,0.9,0.95\}$. For each task and labeled-data configuration, we select the threshold that gives the highest mean best-validation accuracy. The selected threshold, the corresponding test accuracy, and the best-validation accuracy are reported in Table~\ref{tab:app-fixmatch-best}. The unlabeled-loss weight is fixed at $\lambda_U=1.0$. 

For tabular datasets, the weak view is generated by adding Gaussian noise with $\sigma=0.1$. The strong view uses Gaussian noise with $\sigma=0.2$ and feature dropout with drop probability $p=0.1$. For MNIST-family image tasks, including the binary 4-vs-9 and multiclass MNIST settings, the weak view uses a small random affine transform, $\mathrm{degrees}=5$ and translation $(0.05,0.05)$, followed by normalization. The strong view uses a larger random affine transform, $\mathrm{degrees}=15$, translation $(0.1,0.1)$, and scale range $(0.9,1.1)$, followed by random erasing with probability $0.5$ and erased-area scale $(0.02,0.1)$, then normalization. For CIFAR-style image tasks, the weak view uses random horizontal flip and random crop with padding 4; the strong view additionally applies color jitter with brightness/contrast/saturation $0.4$ and hue $0.1$, random grayscale with probability $0.1$, and random erasing with probability $0.5$.

\begin{table}[h]
	\centering
	\footnotesize
	\caption{Classification results for FixMatch. Results are test accuracy and best-validation accuracy in \%, reported as mean$\pm$std over 30 seeds. For binary tasks, Bal. and Imb. denote balanced and imbalanced labeled data setting used in the section \ref{subsec:sslexp}. For MNIST, Bal., Mild, and Sev. follow the regime notation used in the main text.}
	\label{tab:app-fixmatch-best}
	{\setlength{\tabcolsep}{2pt}
		\begin{tabular}{@{}p{0.14\textwidth}p{0.18\textwidth}p{0.08\textwidth}p{0.15\textwidth}p{0.15\textwidth}@{}}
			\toprule
			Task & Regime ($n$) & Threshold & Test acc. (\%) & Best val. acc. (\%) \\
			\midrule
				Adult & Bal.\ (60) & 0.95 & 77.96 $\pm$ 2.11 & 78.61 $\pm$ 4.04 \\
			Adult & Bal.\ (200) & 0.9 & 78.74 $\pm$ 2.03 & 78.61 $\pm$ 2.97 \\	
				Adult & Imb.\ (60) & 0.95 & 78.97 $\pm$ 2.45 & 78.67 $\pm$ 3.70 \\
			Adult & Imb.\ (200) & 0.9 & 79.49 $\pm$ 2.42 & 79.83 $\pm$ 3.20 \\
			Banknote & Bal.\ (60) & 0.9 & 94.88 $\pm$ 3.20 & 95.55 $\pm$ 3.43 \\
			Banknote & Bal.\ (200) & 0.9 & 97.12 $\pm$ 1.95 & 97.61 $\pm$ 2.26 \\
			Banknote & Imb.\ (60) & 0.95 & 93.52 $\pm$ 4.84 & 94.72 $\pm$ 3.16 \\
			Banknote & Imb.\ (200) & 0.95 & 96.46 $\pm$ 2.30 & 97.22 $\pm$ 2.07 \\
					MNIST~4v9 & Bal.\ (60) & 0.8 & 98.83 $\pm$ 0.57 & 99.68 $\pm$ 0.48 \\
MNIST~4v9 & Bal.\ (200) & 0.95 & 98.99 $\pm$ 0.37 & 99.68 $\pm$ 0.40 \\
			MNIST~4v9 & Imb.\ (60) & 0.8 & 98.89 $\pm$ 0.55 & 99.63 $\pm$ 0.47 \\
			MNIST~4v9 & Imb.\ (200) & 0.8 & 99.03 $\pm$ 0.32 & 99.67 $\pm$ 0.48 \\
			Breast Cancer & Bal.\ (60) & 0.8 & 93.22 $\pm$ 2.70 & 96.33 $\pm$ 2.37 \\
			Breast Cancer & Bal.\ (200) & 0.9 & 95.09 $\pm$ 1.99 & 98.33 $\pm$ 1.91 \\
				Breast Cancer & Imb.\ (60) & 0.9 & 93.39 $\pm$ 2.69 & 97.00 $\pm$ 2.64 \\
			Breast Cancer & Imb.\ (200) & 0.95 & 95.00 $\pm$ 2.09 & 98.89 $\pm$ 2.02 \\
			Credit & Bal.\ (60) & 0.95 & 66.97 $\pm$ 7.04 & 68.17 $\pm$ 8.91 \\
				Credit & Bal.\ (200) & 0.95 & 70.18 $\pm$ 7.01 & 71.72 $\pm$ 8.61 \\
			Credit & Imb.\ (60) & 0.95 & 72.89 $\pm$ 4.74 & 76.44 $\pm$ 4.63 \\
			Credit & Imb.\ (200) & 0.8 & 72.22 $\pm$ 7.02 & 73.72 $\pm$ 8.22 \\
			MNIST & Bal.\ (200) & 0.8 & 97.12 $\pm$ 0.54 & 98.00 $\pm$ 0.75 \\
			MNIST & Bal.\ (500) & 0.95 & 96.82 $\pm$ 0.76 & 97.80 $\pm$ 0.98 \\
			MNIST & Bal.\ (1500) & 0.8 & 97.65 $\pm$ 0.55 & 97.95 $\pm$ 1.07 \\
			MNIST & Mild (500) & 0.9 & 96.95 $\pm$ 1.24 & 97.95 $\pm$ 0.80 \\
			MNIST & Mild (1500) & 0.8 & 97.60 $\pm$ 0.53 & 98.10 $\pm$ 0.81 \\
			MNIST & Sev.\ (500) & 0.95 & 94.71 $\pm$ 3.97 & 95.35 $\pm$ 4.47 \\
			MNIST & Sev.\ (1500) & 0.9 & 97.39 $\pm$ 0.50 & 98.00 $\pm$ 0.94 \\
			\bottomrule
		\end{tabular}
	}
\end{table}

\subsection{Covariance-set size effect}
\label{subsec:app-cov-size}
Our method uses a small held-out set to estimate covariance-related quantities. This ablation examines how sensitive \textbf{Ours (Iter)} is to the number of held-out covariance-estimation examples per class. We run the experiment on MNIST under two labeled-data regimes: a balanced regime and a severe-imbalance regime.
Table~\ref{tab:app-cov-size} shows that performance is fairly stable once a modest number of covariance-estimation examples is available. In particular, performance does not collapse even with very small covariance sets, likely helped by diagonal shrinkage and Tikhonov regularization, and largely saturates around 20--50 examples per class.

\begin{table}[h]
	\centering
	\footnotesize
	\caption{MNIST covariance-set size ablation using \textbf{Ours (Iter)}. Results are test accuracy in \%, reported as mean$\pm$std over 30 seeds.}
	\label{tab:app-cov-size}
	{\setlength{\tabcolsep}{4pt}
		\begin{tabular}{@{}lll@{}}
			\toprule
			Regime & Cov. examples per class & Test acc. (\%) \\
			\midrule
			Bal. (500) & 1 & 90.96 $\pm$ 1.08 \\
			Bal. (500) & 2 & 91.11 $\pm$ 2.11 \\
			Bal. (500) & 3 & 91.58 $\pm$ 1.43 \\
			Bal. (500) & 5 & 91.54 $\pm$ 1.36 \\
			Bal. (500) & 10 & 91.95 $\pm$ 0.92 \\
			Bal. (500) & 20 & 92.18 $\pm$ 1.08 \\
			Bal. (500) & 50 & 92.14 $\pm$ 0.70 \\
			Bal. (500) & 100 & 92.28 $\pm$ 0.84 \\
			Bal. (500) & 1000 & 92.55 $\pm$ 0.67 \\
			Sev. (500) & 1 & 80.90 $\pm$ 2.23 \\
			Sev. (500) & 2 & 78.57 $\pm$ 4.85 \\
			Sev. (500) & 3 & 79.57 $\pm$ 3.23 \\
			Sev. (500) & 5 & 80.88 $\pm$ 3.14 \\
			Sev. (500) & 10 & 81.17 $\pm$ 2.76 \\
			Sev. (500) & 20 & 81.75 $\pm$ 2.26 \\
			Sev. (500) & 50 & 82.56 $\pm$ 2.40 \\
			Sev. (500) & 100 & 81.40 $\pm$ 2.06 \\
			Sev. (500) & 1000 & 82.23 $\pm$ 2.61 \\
			\bottomrule
		\end{tabular}
	}
\end{table}

\subsection{Equal-covariance diagnostic}
\label{subsec:app-homoscedasticity}
The proposed data-free method relies on an equal-covariance approximation across classes. This diagnostic quantifies how closely the class-conditional covariance matrices agree with one another. In Table~\ref{tab:app-homoscedasticity}, Cov.\ rel.\ Fro.\ denotes the average pairwise relative Frobenius distance between class covariance matrices:
$$
\frac{\lVert\Sigma_i-\Sigma_j\rVert_F}{(\lVert\Sigma_i\rVert_F+\lVert\Sigma_j\rVert_F)/2}.
$$
Smaller values of this measure indicate that the class-conditional covariance matrices are closer to the equal-covariance approximation. These distances are not exactly zero; hence, equal covariance (EC) should be viewed as a data-free approximation, analogous to the equal-variance simplification used for binary PNU. Nevertheless, our experimental results do not show a clear correlation between the degree of EC violation and performance. One possible reason is that this diagnostic measures covariance discrepancies at fixed trained models, whereas Theorem \ref{thm:gen_bound} suggests the performance is affected by the maximum variance over the hypothesis class. This distinction may explain why this local measure of EC violation is not directly predictive of empirical performance.

\begin{table}[h]
	\centering
	\footnotesize
	\caption{Equal-covariance diagnostic summary. Results are mean$\pm$std over 30 seeds. Smaller Cov.\ rel.\ Fro.\ values indicate closer agreement with the equal-covariance approximation.}
	\label{tab:app-homoscedasticity}
	{\setlength{\tabcolsep}{2pt}
\begin{tabular}{@{}p{0.24\textwidth}p{0.07\textwidth}p{0.14\textwidth}p{0.16\textwidth}@{}}
	\toprule
	Setting & Classes & Test acc. (\%) & Cov. rel. Fro. \\
	\midrule
	Covertype, Bal. (140) & 7 & 56.58 $\pm$ 2.28 & 1.5110 $\pm$ 0.0864 \\
	Covertype, Mild (350) & 7 & 63.71 $\pm$ 1.51 & 1.5344 $\pm$ 0.0847 \\
	Covertype, Sev. (350) & 7 & 64.26 $\pm$ 1.71 & 1.5386 $\pm$ 0.0794 \\
	Dry Beans, Bal. (350) & 7 & 91.56 $\pm$ 0.71 & 1.1433 $\pm$ 0.0275 \\
	Dry Beans, Sev. (140) & 7 & 85.28 $\pm$ 1.94 & 1.1341 $\pm$ 0.1139 \\
	MNIST, Mild (500) & 10 & 89.74 $\pm$ 0.99 & 1.1260 $\pm$ 0.0847 \\
	MNIST, Sev. (500) & 10 & 80.94 $\pm$ 2.98 & 1.0817 $\pm$ 0.0701 \\
	Shuttle, Bal. (70) & 7 & 96.74 $\pm$ 1.57 & 1.3476 $\pm$ 0.0407 \\
	Shuttle, Mild (140) & 7 & 97.88 $\pm$ 0.82 & 1.3545 $\pm$ 0.0314 \\
	\bottomrule
\end{tabular}
	}
\end{table}

\subsection{Non-negative risk-correction ablation}
\label{subsec:app-nn-ablation}

We study the effect of the non-negative risk correction introduced in
Section~\ref{subsec:practicalssl}. Tables~\ref{tab:app-nn-binary-tabular}--\ref{tab:app-nn-severe}
compare \textbf{Ours (Iter)} with and without this correction. In the tables, ``NN on'' is the default method with the non-negative risk
correction, and ``NN off'' is the same method without it. We also report
\[
\Delta = \text{accuracy(NN off)} - \text{accuracy(NN on)}
\]
in percentage points. A negative $\Delta$ means that the correction improves
accuracy, while a positive $\Delta$ means that removing the correction gives
higher accuracy. Overall, the correction has a small effect on many tabular datasets. However, it is highly effective on image datasets with larger accuracy improvements.

\begin{table}[h!]
	\centering
	\footnotesize
	\caption{Non-negative risk-correction ablation on binary tabular tasks. Results are test accuracy in \%, reported as mean$\pm$std over 30 seeds. }
	\label{tab:app-nn-binary-tabular}
	{\setlength{\tabcolsep}{3pt}
		\begin{tabular}{@{}p{0.20\textwidth}p{0.14\textwidth}p{0.20\textwidth}p{0.20\textwidth}p{0.10\textwidth}@{}}
			\toprule
			Task & Labeled & NN on (\%) & NN off (\%) & $\Delta$ (pp) \\
			\midrule
			Adult & $(15,45)$ & 80.3 $\pm$ 1.6 & 80.2 $\pm$ 1.6 & -0.1 \\
			Adult & $(30,30)$ & 80.8 $\pm$ 1.5 & 80.9 $\pm$ 1.1 & +0.1 \\
			Adult & $(50,150)$ & 82.2 $\pm$ 1.0 & 82.1 $\pm$ 0.9 & -0.1 \\
			Adult & $(100,100)$ & 82.4 $\pm$ 1.0 & 82.5 $\pm$ 1.0 & +0.1 \\
			Banknote & $(15,45)$ & 97.3 $\pm$ 1.5 & 96.9 $\pm$ 2.0 & -0.4 \\
			Banknote & $(30,30)$ & 97.8 $\pm$ 1.1 & 97.9 $\pm$ 1.4 & +0.1 \\
			Banknote & $(50,150)$ & 98.6 $\pm$ 0.9 & 98.6 $\pm$ 0.9 & 0.0 \\
			Banknote & $(100,100)$ & 98.8 $\pm$ 1.0 & 98.8 $\pm$ 1.0 & 0.0 \\
			Breast Cancer & $(15,45)$ & 94.5 $\pm$ 2.3 & 94.6 $\pm$ 2.4 & +0.2 \\
			Breast Cancer & $(30,30)$ & 94.4 $\pm$ 1.8 & 94.5 $\pm$ 2.0 & +0.2 \\
			Breast Cancer & $(50,150)$ & 95.5 $\pm$ 1.7 & 95.6 $\pm$ 1.7 & 0.0 \\
			Breast Cancer & $(100,100)$ & 95.6 $\pm$ 1.8 & 95.6 $\pm$ 1.7 & +0.1 \\
			Credit & $(15,45)$ & 78.2 $\pm$ 2.0 & 78.3 $\pm$ 1.8 & +0.1 \\
			Credit & $(30,30)$ & 78.9 $\pm$ 1.2 & 78.9 $\pm$ 1.2 & 0.0 \\
			Credit & $(50,150)$ & 78.9 $\pm$ 1.3 & 79.0 $\pm$ 1.2 & 0.0 \\
			Credit & $(100,100)$ & 79.3 $\pm$ 1.1 & 79.3 $\pm$ 1.1 & 0.0 \\
			\bottomrule
		\end{tabular}
	}
\end{table}
\begin{table}[h!]
	\centering
	\footnotesize
	\caption{Non-negative risk-correction ablation on binary image tasks. Results are test accuracy in \%, reported as mean$\pm$std over 30 seeds.}
	\label{tab:app-nn-binary-image}
	{\setlength{\tabcolsep}{3pt}
		\begin{tabular}{@{}p{0.20\textwidth}p{0.14\textwidth}p{0.20\textwidth}p{0.20\textwidth}p{0.10\textwidth}@{}}
			\toprule
			Task & Labeled & NN on (\%) & NN off (\%) & $\Delta$ (pp) \\
			\midrule
			MNIST~4v9 & $(15,45)$ & 94.7 $\pm$ 1.9 & 92.3 $\pm$ 3.4 & -2.5 \\
			MNIST~4v9 & $(30,30)$ & 95.6 $\pm$ 1.6 & 94.6 $\pm$ 3.1 & -1.0 \\
			MNIST~4v9 & $(50,150)$ & 97.3 $\pm$ 0.9 & 96.9 $\pm$ 1.1 & -0.4 \\
			MNIST~4v9 & $(100,100)$ & 97.7 $\pm$ 0.7 & 97.4 $\pm$ 1.1 & -0.3 \\
			CIFAR~CvD & $(15,45)$ & 55.4 $\pm$ 1.6 & 55.8 $\pm$ 2.1 & +0.4 \\
			CIFAR~CvD & $(30,30)$ & 55.8 $\pm$ 2.0 & 55.2 $\pm$ 2.1 & -0.6 \\
			CIFAR~CvD & $(50,150)$ & 58.2 $\pm$ 2.1 & 58.7 $\pm$ 1.6 & +0.5 \\
			CIFAR~CvD & $(100,100)$ & 59.6 $\pm$ 1.7 & 58.7 $\pm$ 1.7 & -0.9 \\
			\bottomrule
		\end{tabular}
	}
\end{table}
\begin{table}[h]
	\centering
	\footnotesize
	\caption{Non-negative risk-correction ablation on multiclass tasks with equal labeled samples per class. Results are test accuracy in \%, reported as mean$\pm$std over 30 seeds.}
	\label{tab:app-nn-equal}
	{\setlength{\tabcolsep}{3pt}
		\begin{tabular}{@{}p{0.20\textwidth}p{0.14\textwidth}p{0.20\textwidth}p{0.20\textwidth}p{0.10\textwidth}@{}}
			\toprule
			Task & Labeled & NN on (\%) & NN off (\%) & $\Delta$ (pp) \\
			\midrule
			MNIST & 20 & 86.9 $\pm$ 1.3 & 85.6 $\pm$ 1.8 & -1.3 \\
			MNIST & 50 & 91.8 $\pm$ 0.5 & 91.4 $\pm$ 1.3 & -0.4 \\
			MNIST & 150 & 95.4 $\pm$ 0.4 & 95.3 $\pm$ 0.6 & -0.1 \\
			CIFAR-10 & 20 & 31.7 $\pm$ 1.4 & 28.9 $\pm$ 2.1 & -2.9 \\
			CIFAR-10 & 50 & 38.4 $\pm$ 1.2 & 35.0 $\pm$ 2.0 & -3.3 \\
			CIFAR-10 & 150 & 47.1 $\pm$ 1.1 & 44.3 $\pm$ 1.2 & -2.7 \\
			Covertype & 10 & 52.0 $\pm$ 2.9 & 53.1 $\pm$ 3.1 & +1.2 \\
			Covertype & 20 & 55.8 $\pm$ 1.3 & 56.3 $\pm$ 2.8 & +0.6 \\
			Covertype & 50 & 61.2 $\pm$ 2.1 & 61.4 $\pm$ 1.9 & +0.2 \\
			Shuttle & 10 & 96.7 $\pm$ 1.2 & 96.6 $\pm$ 1.5 & -0.1 \\
			Shuttle & 20 & 97.3 $\pm$ 1.0 & 97.5 $\pm$ 0.9 & +0.2 \\
			Shuttle & 50 & 97.3 $\pm$ 1.0 & 97.8 $\pm$ 0.9 & +0.5 \\
			Dry Beans & 10 & 87.6 $\pm$ 1.6 & 87.5 $\pm$ 1.7 & -0.1 \\
			Dry Beans & 20 & 89.9 $\pm$ 0.9 & 89.6 $\pm$ 1.1 & -0.3 \\
			Dry Beans & 50 & 91.3 $\pm$ 0.6 & 91.3 $\pm$ 0.6 & 0.0 \\
			\bottomrule
		\end{tabular}
	}
\end{table}
\begin{table}[h]
	\centering
	\footnotesize
	\caption{Non-negative risk-correction ablation on multiclass tasks with mild imbalance. Results are test accuracy in \%, reported as mean$\pm$std over 30 seeds. }
	\label{tab:app-nn-mild}
	{\setlength{\tabcolsep}{3pt}
		\begin{tabular}{@{}p{0.20\textwidth}p{0.14\textwidth}p{0.20\textwidth}p{0.20\textwidth}p{0.10\textwidth}@{}}
			\toprule
			Task & Labeled & NN on (\%) & NN off (\%) & $\Delta$ (pp) \\
			\midrule
			MNIST & 500 & 89.8 $\pm$ 1.4 & 89.3 $\pm$ 1.8 & -0.5 \\
			MNIST & 1500 & 94.3 $\pm$ 0.8 & 94.0 $\pm$ 0.9 & -0.3 \\
			CIFAR-10 & 500 & 34.7 $\pm$ 1.0 & 33.2 $\pm$ 1.9 & -1.5 \\
			CIFAR-10 & 1500 & 42.6 $\pm$ 1.0 & 42.2 $\pm$ 2.1 & -0.4 \\
			Covertype & 140 & 57.4 $\pm$ 2.1 & 58.5 $\pm$ 2.4 & +1.0 \\
			Covertype & 350 & 62.7 $\pm$ 1.4 & 62.5 $\pm$ 1.1 & -0.2 \\
			Shuttle & 140 & 96.9 $\pm$ 1.0 & 97.3 $\pm$ 1.5 & +0.4 \\
			Shuttle & 350 & 97.6 $\pm$ 1.1 & 97.7 $\pm$ 1.0 & +0.1 \\
			Dry Beans & 140 & 89.1 $\pm$ 1.7 & 88.9 $\pm$ 1.8 & -0.2 \\
			Dry Beans & 350 & 90.7 $\pm$ 0.9 & 90.6 $\pm$ 1.0 & -0.1 \\
			\bottomrule
		\end{tabular}
	}
\end{table}
\begin{table}[h]
	\centering
	\footnotesize
	\caption{Non-negative risk-correction ablation on multiclass tasks with severe imbalance. Results are test accuracy in \%, reported as mean$\pm$std over 30 seeds. }
	\label{tab:app-nn-severe}
	{\setlength{\tabcolsep}{3pt}
		\begin{tabular}{@{}p{0.20\textwidth}p{0.14\textwidth}p{0.20\textwidth}p{0.20\textwidth}p{0.10\textwidth}@{}}
			\toprule
			Task & Labeled & NN on (\%) & NN off (\%) & $\Delta$ (pp) \\
			\midrule
			MNIST & 500 & 81.1 $\pm$ 1.1 & 74.9 $\pm$ 3.1 & -6.2 \\
			MNIST & 1500 & 89.5 $\pm$ 1.7 & 86.9 $\pm$ 2.3 & -2.6 \\
			CIFAR-10 & 500 & 28.1 $\pm$ 1.3 & 25.6 $\pm$ 2.1 & -2.5 \\
			CIFAR-10 & 1500 & 34.1 $\pm$ 1.9 & 33.4 $\pm$ 1.8 & -0.7 \\
			Covertype & 140 & 58.0 $\pm$ 1.7 & 59.0 $\pm$ 2.2 & +0.9 \\
			Covertype & 350 & 63.3 $\pm$ 1.7 & 63.4 $\pm$ 1.6 & +0.1 \\
			Shuttle & 140 & 97.2 $\pm$ 0.9 & 97.3 $\pm$ 1.1 & +0.1 \\
			Shuttle & 350 & 97.6 $\pm$ 0.8 & 97.6 $\pm$ 0.8 & +0.1 \\
			Dry Beans & 140 & 84.9 $\pm$ 2.6 & 85.5 $\pm$ 2.4 & +0.6 \\
			Dry Beans & 350 & 89.1 $\pm$ 1.5 & 88.9 $\pm$ 1.5 & -0.1 \\
			\bottomrule
		\end{tabular}
	}
\end{table}